\documentclass{article}

\usepackage{microtype}
\usepackage{graphicx}

\usepackage{booktabs} 

\usepackage{multirow}   
\usepackage{subfig}
\usepackage{caption}

\usepackage{soul}
\usepackage{hyperref}

\usepackage{enumitem}

\usepackage[accepted]{icml2024}

\usepackage{amsmath}
\usepackage{amssymb}
\usepackage{mathtools}
\usepackage{amsthm}

\usepackage[capitalize,noabbrev]{cleveref}

\theoremstyle{plain}
\newtheorem{theorem}{Theorem}[section]

\theoremstyle{definition}

\theoremstyle{remark}

\usepackage[textsize=tiny]{todonotes}

\icmltitlerunning{Higher-order Spatio-temporal Physics-incorporated Graph Neural Network for Multivariate Time Series Imputation}

\begin{document}

\twocolumn[
\icmltitle{Higher-order Spatio-temporal Physics-incorporated Graph Neural Network for Multivariate Time Series Imputation}



\icmlsetsymbol{equal}{*}

\begin{icmlauthorlist}
\icmlauthor{Guojun Liang}{sch}
\icmlauthor{Prayag Tiwari}{sch}
\icmlauthor{Slawomir Nowaczyk}{sch}
\icmlauthor{Stefan Byttner}{sch}

\end{icmlauthorlist}


\icmlaffiliation{sch}{School of Information Technology, Halmstad University, Halmstad, Sweden}

\icmlcorrespondingauthor{Guojun Liang}{guojun.liang@hh.se}
\icmlcorrespondingauthor{Prayag Tiwari}{prayag.tiwari@hh.se}


\vskip 0.3in
]


%
\printAffiliationsAndNotice{}  

\begin{abstract}
Exploring the missing values is an essential but challenging issue due to the complex latent spatio-temporal correlation and dynamic nature of time series. Owing to the outstanding performance in dealing with structure learning potentials, Graph Neural Networks (GNNs) and Recurrent Neural Networks (RNNs) are often used to capture such complex spatio-temporal features in multivariate time series. However, these data-driven models often fail to capture the essential spatio-temporal relationships when significant signal corruption occurs. Additionally, calculating the high-order neighbor nodes in these models is of high computational complexity. To address these problems, we propose a novel higher-order spatio-temporal physics-incorporated GNN (HSPGNN). Firstly, the dynamic Laplacian matrix can be obtained by the spatial attention mechanism. Then, the generic inhomogeneous partial differential equation (PDE) of physical dynamic systems is used to construct the dynamic higher-order spatio-temporal GNN to obtain the missing time series values.
Moreover, we estimate the missing impact by Normalizing Flows (NF) to evaluate the importance of each node in the graph for better explainability. Experimental results on four benchmark datasets demonstrate the effectiveness of HSPGNN and the superior performance when combining various order neighbor nodes. Also, graph-like optical flow, dynamic graphs, and missing impact can be obtained naturally by HSPGNN, which provides better dynamic analysis and explanation than traditional data-driven models.  Our code is available at \href{https://github.com/gorgen2020/HSPGNN}{\textit{https://github.com/gorgen2020/HSPGNN}}.
\end{abstract}

\section{Introduction}
Due to the unreliability of sensors and the malfunctions of the transmitting network, missing data is very common in the time series data collection stage. Thus, time series imputation is a crucial task in multivariate time series analysis (TSA) \citep{jin2023survey}, but obtaining accurate values for the missing data is a challenging task. In the early research stages, scholars focused on capturing the temporal structure and ignored the abundant latent spatial structure of multivariate time series. Notably, the spatial structure sometimes does not simply represent physical proximity but the functional dependency of sensors or their distance in latent space \citep{ziat2017spatio}. Recently, researchers have realized the importance of spatial structure in multivariate time series and developed many spatio-temporal models in this field. Among these models, Deep Learning (DL) models are the most promising approaches for their outstanding non-linear fitting and representation capacity. To capture the temporal feature, recurrent neural networks (RNNs) \citep{li2017diffusion, cao2018brits,liang2023semantics,challu2023nhits} and temporal attention \citep{vaswani2017advances,guo2020dynamic,du2023saits} are often adopted. The core concept of RNN is using the recurrent learning structure for time series, which allows the output from some nodes to affect subsequent input to the same nodes \citep{tealab2018time}. Graph neural networks (GNNs) hold great potential in dealing with non-Euclidean structure data and have become a hotspot in this field \citep{asif2021graph,wu2021inductive,li2017diffusion,andrea2021filling}. However, many researchers applied GNNs with the static graph of the first-order neighbors, which cannot fully exploit the dynamics of latent spatial structure and the higher-order neighbor nodes. Moreover, spatio-temporal features are difficult to capture by those data-driven models in the case of the partially observed datasets with significant missing rates or various missing patterns \citep{chen2019missing}. In addition, these models lack a strong explanation of the imputation mechanism in GNN, which in turn limits the further application of these models.

Vanilla RNNs are used to learn temporal derivatives \citep{elman1990finding} but lack strong theoretical analyses and require high time complexity. Spatio-temporal features are difficult to capture by those data-driven models in the case of the partially observed datasets with significant missing rates or various missing patterns \citep{chen2019missing}.  In addition, these models lack a strong explanation of the imputation mechanism in GNN, which in turn limits the further application of these models.

Meanwhile, physics models are applied to deal with other time series problems. Physics models can be representative of the real-world dynamics of time series such that the unobserved values can be properly predicted with small data \citep{shi2021physics} and provide better explainability. Many physics models are utilized in various time series applications, e.g., the Lighthill- Whitham-Richards (LWR) model \citep{lighthill1955kinematic} in traffic state estimation, the wave propagation in heat diffusion system \citep{tie2016theoretical}, and the air leakage physics model in vehicles \citep{fan2022incorporating}. However, there are two main disadvantages of physics models. For one thing, choosing the proper physics model is quite difficult since different physics laws govern the dynamic systems in different domains, or even one dynamic system is governed by several physics laws with partial differential equations (PDEs) of different order derivatives. Therefore, inappropriately chosen models may result in poor performance. For another thing, the physical parameters or the internal and external sources are unknown in some applications, which leads to infinite possible solutions in the physics models. Paradoxically, to justify the model requires a large amount of observed data, which is impossible in missing severe data situations. Recently, physics-informed neural network (PINN) models have been proposed to attempt to solve the second problem \citep{cai2021physics,raissi2019physics,cuomo2022scientific}. The framework combines the physics PDE terms as part of the optimization process via regularization, which intends to mitigate the limitations of DL and the physics model. Unfortunately, while it is demonstrated that existing PINN methodologies can learn good models for relatively trivial problems, they can easily fail to learn relevant physical phenomena for even slightly complex problems \citep{krishnapriyan2021characterizing}. These possible failures are due to the PINN’s setup, which makes the loss landscape very hard to optimize. To address the problems mentioned above, we propose a novel Higher-order Spatio-temporal Physics-incorporated GNN (HSPGNN) framework for multivariate time series imputation. The main contributions are as follows:
\begin{itemize}[leftmargin=*] 
\item We incorporate the physics model into a data-driven model with trainable physical parameters, where the novel model can optimize the combination of different physics laws with different orders, making our model more robust and able to provide better explanations. Also, we estimate the missing impact by Normalizing Flows (NF) to evaluate the importance of each node in the graph. 

\item We prove that the space and time complexity of vanilla RNN in calculating the $(M-1)$-order derivative by stacking layers for $M$ length of $N$ nodes time series is at least $\mathcal{O}(3M(M-1)/2)\approx \mathcal{O}(M^2) $ and $\mathcal{O}(M^2-M)MN/2)\approx \mathcal{O}(M^3N)$, while our model can reduce the space and time complexity in calculating the combination of $(M-1)$-order derivative to approximately $\mathcal{O}(k_t)$ and $\mathcal{O}(k_tM^2N)$ with $k_t \ll M$, respectively.     
\item Experimental results on four benchmark datasets show the effectiveness and explainability of HSPGNN, especially when the data are of complex missing patterns. Also, the combination of multi-hop is better than only one specific hop, which shows the existence of multi-hop physical correlations in the time series data.  
\end{itemize}

\section{Related work}


\textbf{Matrix completion models:} Taking the time series dataset as a matrix or tensor, another data imputation in time series is named matrix completion  \citep{johnson1990matrix,candes2012exact,nguyen2019low}. The main approach utilizes matrix factorization (MF) to obtain the low-rank representation and then use it to recover the missing data. Traditional MF \citep{cichocki2009fast} methods adopt the singular value decomposition (SVD) that calculates low-rank representations and use them to recover the missing data. TTLRR \citep{yang2022robust} proposes a novel transformed tensor SVD for tensor data recovery and subspace clustering. These methods impute the missing data under a solid theoretical guarantee for data recovery. However, they take the time series as a grid structure dataset, neglecting the spatio-temporal correlation hidden in the matrix. M$^2$DMTF \citep{fan2021multi}  proposes a new tensor decomposition method called multi-mode nonlinear deep tensor factorization, which can be regarded as a high-order generalization of the two-mode nonlinear deep MF, but the optimization efficiency is low in the complex data missing situation. HRST-LR \citep{xu2023hrst} attempts to represent the spatio-temporal correlation by the multiplication of the spatial matrix and the temporal matrix.  Still, the spatio-temporal correlation is far more complex than the formulated objective function. Therefore, these models limit the further improvement.

\textbf{Non-GNN deep learning models:}
Many scholars attempt to obtain the missing value of higher accuracy through deep learning methods \citep{fang2020time}. Some of them focus on representing the temporal feature of time series. RNN and attention mechanisms are often applied to capture the temporal features. BRITS \citep{cao2018brits} proposes a bidirectional RNN for imputation. Transformers \citep{vaswani2017advances} try to capture the temporal feature by the parallel multi-head attention mechanism. SAITS \citep{du2023saits} adopts a weighted combination of two diagonally-masked self-attention for imputation. NHITS \citep{challu2023nhits} incorporates novel hierarchical interpolation and multi-rate data sampling techniques for time series prediction and interpolation. TDM \citep{zhao2023transformed} imputes the missing values of two batches of samples by transforming them into a latent space through deep invertible functions and matching them distributionally. However, these models ignore the latent graph structure of the multivariate time series.

\textbf{GNN models:}
More and more scholars are gradually realizing the spatial correlations between the sensors since most time series datasets have innate and inherent graph structures. Such a missing data phenomenon in the graph is known as a spatio-temporal kriging problem. Traditional machine learning methods, such as KNN \citep{cover1967nearest, oehmcke2016knn,sridevi2011imputation}, try to find out the missing value by the neighbor nodes, but its accuracy depends on feature engineering, and the representation capability is limited. GNN is a neural network devised to deal with graph representation in neural networks and is a hotspot in this field. IGNNK \citep{wu2021inductive} model generates random subgraphs as samples and the corresponding adjacency matrix for each sample to reconstruct the missing data. MPGRU, similar to the traffic flow forecasting model DCRNN \citep{li2017diffusion}, applies the GNN and Gated Recurrent Unit (GRU) to capture the spatio-temporal features for imputation. GRIN \citep{andrea2021filling} proposes a bidirectional GNN and GRU framework to learn spatio-temporal representations by message passing mechanism. However, even though some of those data-driven models mentioned above can obtain high accuracy in some datasets, their robustness can easily be affected by the missingness pattern. Also, it is difficult for these data-driven models to explore the physics correlations hidden in the datasets, especially when the datasets are seriously corrupted.

\textbf{Explainable GNN:}
In another research hot field of explainable GNN, one distinguishes instance-level explanations and model-level explanations \citep{yuan2022explainability}. The instance explainable GNNs attempt to explain the models by the relationship between the output of GNN and the input node or the input features of the node. SA \citep{baldassarre2019explainability} produces local explanations using the squared norm of its gradient w.r.t the input features. GNNExplainer \citep{ying2019gnnexplainer} maximizes the mutual information between a GNN's prediction and the distribution of possible subgraph structures to explain the predictions. GraphMask \citep{schlichtkrull2020interpreting} proposes a post-hoc method for interpreting the predictions of GNNs which identifies unnecessary edges. GraphLime \citep{huang2022graphlime} leverages the feature information of the N-hop network neighbors of the node being explained and their predicted labels in a local subgraph, which produces finite features as the explanations for a particular prediction. GNN-LRP \citep{schnake2021higher} produces detailed explanations that subsume the complex nested interaction between the GNN model and the input graph. The model-level method provides a generic understanding of GNN based on graph generation. XGNN \citep{yuan2020xgnn} explains GNNs by training a graph generator so that the generated graph patterns maximize a certain prediction of the model. However, None of the models mentioned above explain the models from the physics perspective.  Moreover, 
these explainable GNN models depend on the completion of the datasets. Therefore, it is difficult to satisfy this condition when the original datasets exhibit significant levels of missing data.

\section{Methodology}
\textbf{Problem Definition}: For multivariate time series of N sensors with T time steps and D features in each sensor, it can be denoted as $\mathbf{X} \in \mathbb{R}^{T \times N \times D}$. $\mathbf{X}_t \in \mathbb{R}^{N \times D}$ represents the feature matrix at time $t$. Obviously, $\mathbf{X} = \{\mathbf{X}_1,\mathbf{X}_2,\dots, \mathbf{X}_T\}$. Correspondingly, $\mathbf{M} \in \mathbb{R}^{T \times N \times D}$ is used to represent the mask of missing value. $\mathbf{M}_{ijk}=1$ means that the data in the $k$-th feature of $j$-th sensor at $i$-th time step is missing. Otherwise,  $\mathbf{M}_{ijk}=0$ represents no missing data. The main task of time series imputation is to predict the missing value in time series. It can be formulated as follows: 
\begin{equation} 
\left [ \mathbf{X} \odot (1-\mathbf{M}) \right ]\xrightarrow{f(.)} \left [ \mathbf{X} \odot \mathbf{M} \right ],  
\label{Eq:problem definition1}  
\end{equation} 

where $\odot$ denotes the Hadamard product. Our goal is to learn a function $f(.)$ that maps the observed values to predict the missing values. In this study, the time series from each sensor are always univariate, thus, $D=1$ throughout the following discussion.

The overview of HSPGNN is shown in Fig. \ref{framework}. For each input missing data $\mathbf{X}_{t-M:t}$ of $M$ length batch, first a multilayer perceptron (MLP) is applied to get the coarse imputation value $\mathbf{\Bar{P}}_{t-M:t}$. Then, the dynamic Laplacian matrices $L_{t-M:t}$ are obtained using spatial attention with these coarse imputed values. Two physics-incorporated layers are adopted to calculate the combination of the high-order Laplacian matrix and Toeplitz matrix with physics constraint and get the more precise imputation value $\mathbf{\hat{P}}_{t-M:t}$. Long short-term memory (LSTM) and temporal attention are adopted to predict the future data $\mathbf{\hat{X}}_{t:t+M} \odot (1-\mathbf{M}_{t:t+M})$ without missing since the ground truth values of missing data are not allowed to take part in the training stage.  
\begin{figure*}[!h] 
\centering  
\includegraphics[width=\textwidth]{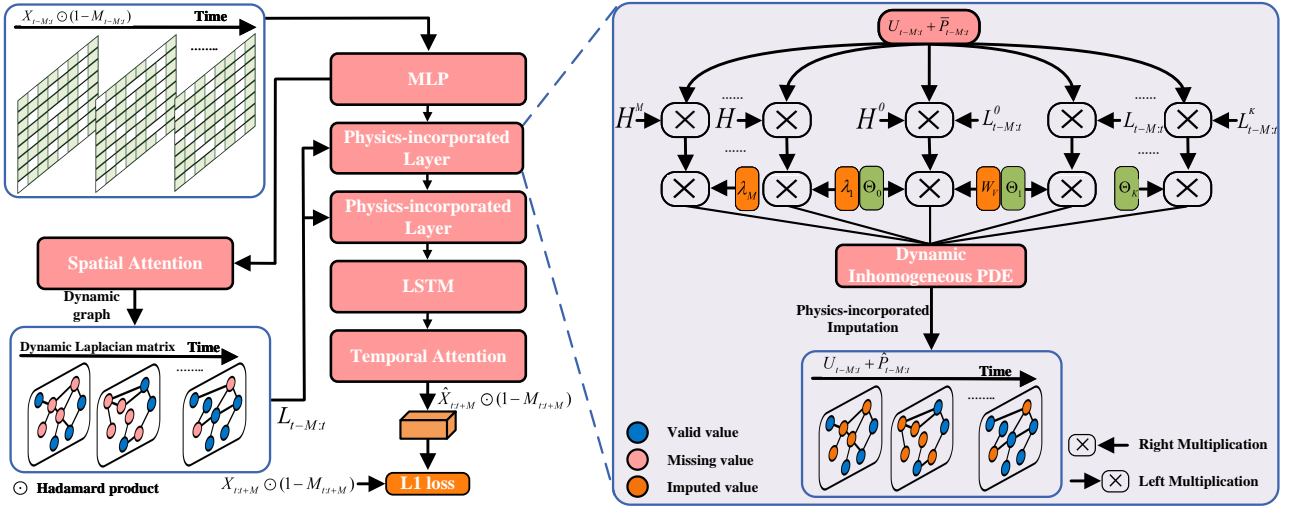} 
\caption{The framework of HSPGNN. Firstly, coarse imputation value  $\mathbf{\Bar{P}}$ can be estimated by MLP. Then, on the left, the latent dynamic spatial graphs $\boldsymbol{\mathcal{G}_t}$ structure is learned from the estimated value, providing the dynamic Laplacian matrix. On the right, a novel physics-incorporated GNN layer is devised, which can combine multi-order spatial and temporal neighbor signals into consideration. As for the physics-incorporated GNN layer, physical dynamic system inhomogeneous PDE is applied to capture the spatio-temporal correlation and perform more accurate imputed values $\hat{\mathbf{P}}_{t-M:t}$. Thus, the imputed values are utilized to predict the future value $\hat{\mathbf{X}}_{t:t+M} \odot (1-\mathbf{M}_{t:t+M})$ by LSTM and temporal attention since the ground truth of the missing values are not available in the training stage.}  \label{framework}   
\end{figure*}

\subsection{Generic inhomogeneous PDE of physical dynamic system}
Considering the 2D physical dynamic system of continuous time, the spatio-temporal constraint with the time-varying external sources can be formulated as the following generic inhomogeneous PDE \citep{salsa2022partial,farlow1993partial, saha2021physics}: 
\begin{equation}
\label{PDE} 
\begin{aligned}     
\frac{\partial^m u_t}{\partial t^m}=F(x,y,u_t,\frac{\partial u_t}{\partial x},\frac{\partial u_t}{\partial y},\frac{\partial^2 u_t}{\partial x^2},\frac{\partial^2 u_t}{ \partial y^2},...;\theta)\\
+v_t(x,y),  \quad  (x,y)\in \Omega,\quad t \in [0,T],  
\end{aligned}
\end{equation}
where $u_t(x,y)$ is the observed physical value at the spatial location $(x,y)$ at time $t$, $F$ is the spatial function of $u_t$, and $\theta$ are the parameters of $F$. Moreover, $v_t(x,y)$ is the unobservable external source term or perturbation at location $(x,y)$ at time $t$, which is governed by independent dynamics. This equation is widespread in many applications by considering a specific order. For example, when considering the first orders of $u_t$ and ignoring or approximating the unobserved term, the equation is equivalent to the LWR model in traffic state estimation, the optical flow \citep{aslani2013optical} in tracking moving objects in the video, and so on. Also, the wave propagation equation in the heat diffusion system \citep{o1994wave} can be obtained when considering the second order of the equation. 

However, the dynamic physical PDE is difficult to solve since the external source is unknown, and so is the $\theta$ in the right-hand term. Moreover, the equation will probably be a combination of multi-order and governed by several physics laws in the complex spatio-temporal correlation situation, which makes it challenging to find the PDE solution. On the other hand, some scholars utilize PINN to find the optimal solution, taking advantage of the excellent non-linear fitting ability of neural networks. PINNs often consider the physics law as a soft regularization on an empirical loss function, but these approaches can easily fail to learn the relevant physical phenomena for complex problems. This challenge mainly lies in optimization caused by soft regularization. While \citet{krishnapriyan2021characterizing}   proposed curriculum regularization, it is time-consuming and laborious to deploy it. Moreover, the curriculum regularization must restart when the physics PDE constraint changes. To avoid such a problem, instead of serving the physics PDE as the soft regularization, we incorporate the physics law into our neural network. Firstly, we assume the dynamic system as (1) The function $F$ can be represented by the linear combination of different orders of $u_t$. (2) The unknown external source is related to the current state of $u_t$ and can be represented by the linear relationship as $v_t(x,t)=w_t*u_t$  

Also, the left-hand term in Eq. \ref{PDE} can be regarded as the combination of mixed time order with the learnable parameter $\lambda$  since the $n$ order is unknown and the dynamic system can be governed by several physical laws of different partial derivative order in time $t$. As a result, Eq. \ref{PDE} can be transformed as follows:
\begin{equation} 
\label{sPDE} 
\begin{aligned}   
\sum_{m=1}^{M'} \lambda_m\frac{\partial^m u_t}{\partial t^m}=\sum_{k_x=0}^K\sum_{k_y=0}^K \theta_{k_x}\theta_{k_y}\frac{\partial^{k_x} u_t}{\partial x^{k_x}}\frac{\partial^{k_y} u_t}{\partial y^{k_y}}\\
+w_v u_t,  \quad  (x,y)\in \Omega,\quad t \in [0,T].   
\end{aligned}  
\end{equation}
In this study, we focus on discrete multivariate time series imputation with the non-Euclidean graph structure. The first-order derivative of time can be defined as follows: 
\begin{equation}     
\frac{\partial\mathbf{X}_{t-1}}{\partial t} = \frac{\mathbf{X}_t - \mathbf{X}_{t-1}}{t-(t-1)}  =\Delta \mathbf{X}_t.  
\label{RNN22} 
\end{equation} 
Then, the first-order derivative of space at time $t$ can be defined as: 
\begin{equation}     
\frac{\partial \mathbf{X}_t^{i}}{\partial x}=\sum_{j \in N(i)}\mathbf{X}_t^{j}  - \mathbf{X}_t^{i}. \end{equation} 
where $\mathbf{X}_t^{i}$ represent the feature of $i$-th node at time t, while $N(i)$ represent the neighbors of  $i$-th node. Thus, From the definitions of the Toeplitz matrix $\mathbf{H}=Toeplitz(0,1,-1)$ \citep{bottcher2000toeplitz} and Laplacian matrix $\mathbf{L}$ \cite{kipf2016semi}, let $\Delta \mathbf{X}_1=\mathbf{X}_1$, the matrix format of first-order derivative of time and space can be formulated as follows: 
\begin{equation}  
\scalebox{0.75}{$
\frac{\partial\mathbf{X}}{\partial t} =\begin{bmatrix} \Delta \mathbf{X}_1 \\ \Delta \mathbf{X}_2 \\ \dots \\\Delta \mathbf{X}_M  \end{bmatrix}= \begin{bmatrix}  0 & 0  & \dots  & 1 \\  0 & \dots & 1 & -1  \\  0 & \dots & \dots  & \dots   \\  0 & 1 & -1 & \dots \\   1 & -1 & 0 & \dots  \end{bmatrix}_{M \times M} \otimes  \begin{bmatrix} \ \mathbf{X}_M \\ \mathbf{X}_{M-1} \\ \dots \\ \mathbf{X}_1  \end{bmatrix}_{ M \times N} = \mathbf{HX}      $}
\label{RNN1}
\end{equation}

\begin{equation}   \scalebox{0.70}{$
\frac{\partial\mathbf{X}}{\partial x}=\begin{bmatrix} \sum_{j \in N(1)}\mathbf{X}^{j}  - \mathbf{X}^{1}, \sum_{j \in N(2)}\mathbf{X}^{j}  - \mathbf{X}^{2}, \dots , \sum_{j \in N(N)}\mathbf{X}^{j}  - \mathbf{X}^{N} \end{bmatrix}= -\mathbf{XL} $}
\end{equation}

where $\otimes$ denotes matrix multiplication. Therefore, the Eq. \ref{sPDE} can be formulated as the discrete format with the feature matrix as follows:
\begin{equation} 
\label{matrix}
\begin{aligned} 
\sum_{m=1}^{M'} \lambda_m \mathbf{H}^{m}\mathbf{X}=\sum_{k=0}^K \mathbf{\Theta}_{k}\mathbf{X}\mathbf{L}^k+\mathbf{X}\mathbf{W_v} 
\end{aligned}   
\end{equation}
Eq. \ref{matrix} is the matrix format of the dynamic system. However, when there is missing data in the matrix $\mathbf{X}$, the equation is no longer satisfied with only observed values. Considering the partially observed value $\mathbf{U}=\mathbf{X} \odot (1-\mathbf{M})$ and the missing value $\mathbf{P}=\mathbf{X} \odot \mathbf{M}$, Eq. \ref{matrix} can be transformed as follows: 
\begin{equation} 
\label{missmatrix}
\scalebox{0.75}{$
\begin{aligned}   \sum_{m=1}^{M'} \lambda_m \mathbf{H}^{m}(\mathbf{U} + \mathbf{P}) =\sum_{k=0}^K \mathbf{\Theta}_{k}(\mathbf{U} + \mathbf{P})\mathbf{L}^k + (\mathbf{U} + \mathbf{P})\mathbf{W_v} 
\end{aligned}  $}
\end{equation}

It is still challenging to find out the precise missing value in the heavy missing data situation even though we obtained the Eq. \ref{missmatrix}. Instead, we apply a simple MLP to impute the coarse imputation value  $\mathbf{\Bar{P}} = \mathrm{MLP}(\mathbf{U})$, then, inspired by alternating direction method of multipliers (ADMMs) optimization method in matrix completion, these coarse values are input to Eq. \ref{missmatrix} for training to get an estimated value $\hat{\mathbf{P}}_t$ with more accuracy at time $t$ than $\bar{\mathbf{P}}_t$. Thus, Eq. \ref{missmatrix} can be converted into vector format with the coarse imputation $\mathbf{\Bar{P}}_t$ as follows: \begin{equation}    
\label{tmatrix} 
\begin{aligned} \mathbf{\hat{P}}_{t}=\sum_{k=0}^K \mathbf{\Theta_{k} }(\mathbf{U}_{t-1} + \mathbf{\Bar{P}}_{t-1})\mathbf{L}_t^k + (\mathbf{U}_{t-1} + \mathbf{\Bar{P}}_{t-1})\mathbf{W_v} \\
- \sum_{m=1}^{M'-2}{\lambda}'_m(\mathbf{U}_{t+m} + \mathbf{\Bar{P}}_{t+m}) -{\lambda}'_0(\mathbf{U}_{t-1} + \Bar{\mathbf{P}}_{t-1} ),
\end{aligned}  
\end{equation}
where the$ \mathbf{\lambda}'=[\lambda_{0}',\lambda_{1}',\cdots,\lambda'_{M'-2}]$ are the parameters of $\sum_{m=1}^{M'} \lambda_m \mathbf{H}^{m}$ after the combination of row vectors. Also, we adopt the spatial attention (SAtt) \citep{feng2017effective, guo2019attention} to generate the dynamic Laplacian matrix $\mathbf{L}_t$ instead of the static graph.

\subsection{Representational capability analysis}   
 Due to the creative recurrent structure for time series, RNNs are often adopted to capture the temporal feature. 
 In some PINN models \citep{saha2021physics}, they try to learn the temporal derivative by vanilla RNN \citep{elman1990finding}. However, the space and time complexity in learning the derivative by vanilla RNN is high.    
 \begin{theorem}    
 The vanilla RNNs of single layer must have more than $M-2$ neurons to calculate the first-order derivative of $M$ length time series. 
 \label{theorem1}
 \end{theorem}

\begin{proof}
The vanilla RNN can be described as follows: \begin{equation}
\begin{aligned}    
h_t &=\sigma_h(W_hx_t+u_hh_{t-1}+b_h)\\        y_t &=\sigma_y(w_y h_t+b_y)       
\end{aligned}
\label{RNN}
\end{equation}  
 From Eq. \ref{RNN1}, the derivative has a linear relationship. Thus, let $\sigma_h(x_t)=x_t$, $\sigma_y(x_t)=x_t$, and ignore the bias term. To simplify the proof, we take the initial hidden state $h_0=0$ in vanilla RNN, and the output of vanilla RNN to calculate the first-order partial derivative of time series. Thus, Eq. \ref{RNN22} can be formulated as:  
\begin{equation} 
\mathbf{y}_t=\frac{\partial\mathbf{X}_{t-1}}{\partial t} = \mathbf{X}_t - \mathbf{X}_{t-1}=\Delta \mathbf{X}_t,\quad t \ge 2 
\label{RNN2} 
\end{equation}
Considering only one neuron in the RNN and Eq. \ref{RNN}, the matrix format of Eq. \ref{RNN2} can be formulated as:  
\begin{equation}     
\scalebox{0.85}{$
\begin{bmatrix} y_2 \\ y_3 \\ \dots \\ y_M  \end{bmatrix} =w_y w_h \begin{bmatrix}  0 & 0  & \dots  & 1 & u_h\\  0 & \dots & 1 & u_h & u_h^2 \\  0 & \dots & \dots  & \dots  & \dots \\  0 & 1 & u_h & \dots & u_h^{M-2}\\   1 & u_h & u_h^2 & \dots & u_h^{M-1} \end{bmatrix}_{(M-1) \times M} \otimes  \begin{bmatrix} \ \mathbf{X}_M \\ \mathbf{X}_{M-1} \\ \dots \\ \mathbf{X}_1  \end{bmatrix} $}    \label{RNN3}  \end{equation}  
From Eq. \ref{RNN1} and Eq. \ref{RNN3}, that equals to find out the solution to the following equation:  \begin{equation}  
\scalebox{0.85}{$ w_y w_h\begin{bmatrix}  0 & 0  & \dots  & 1 & u_h\\  0 & \dots & 1 & u_h & u_h^2 \\  0 & \dots & \dots  & \dots  & \dots \\  0 & 1 & u_h & \dots & u_h^{M-2}\\   1 & u_h & u_h^2 & \dots & u_h^{M-1} \end{bmatrix} = \begin{bmatrix}  0 & 0  & \dots  & 1 & -1\\  0 & \dots & 1 & -1 & 0 \\  0 & \dots & \dots  & \dots  & 0 \\  0 & 1 & -1 & \dots & 0\\   1 & -1 & 0 & \dots & 0  \end{bmatrix} $}
\label{RNN4}  
\end{equation} 
From Cramer's Rule \citep{strang2022introduction}, there is no solution of $u_h$ when $M>2$. Considering the $m$ neurons in RNN, the Eq. \ref{RNN4} can be formulated as follows:  
\begin{equation}  
\begin{aligned} 
\scalebox{0.50}{$
\begin{bmatrix}  0 & 0  & \dots  & \sum_{l=1}^{m}w_y^{(l)} w_h^{(l)}& \sum_{l=1}^{m}w_y^{(l)} w_h^{(l)}u_h^{(l)}\\    0 & \dots & \sum_{l=1}^{m}w_y^{(l)} w_h^{(l)} & \sum_{l=1}^{m}w_y^{(l)} w_h^{(l)}u_h^{(l)}& \sum_{l=1}^{m}w_y^{(l)} w_h^{(l)}(u_h^{(l)})^2 \\    0 & \dots & \dots  & \dots  & \dots \\     0 & \sum_{l=1}^{m}w_y^{(l)} w_h^{(l)} & \sum_{l=1}^{m}w_y^{(l)} w_h^{(l)}u_h^{(l)} & \dots & \sum_{l=1}^{m}w_y^{(l)} w_h^{(l)}(u_h^{(l)})^{M-2}\\    \sum_{l=1}^{m}w_y^{(l)} w_h^{(l)} & \sum_{l=1}^{m}w_y^{(l)} w_h^{(l)}u_h^{(l)} & \sum_{l=1}^{m}w_y^{(l)} w_h^{(l)}(u_h^{(l)})^2 & \dots & \sum_{l=1}^{m}w_y^{(l)} w_h^{(l)}(u_h^{(l)})^{M-1} \end{bmatrix} $} \\  \scalebox{0.85}{$ =  \begin{bmatrix}  0 & 0  & \dots  & 1 & -1\\  0 & \dots & 1 & -1 & 0 \\  0 & \dots & \dots  & \dots  & 0 \\  0 & 1 & -1 & \dots & 0\\   1 & -1 & 0 & \dots & 0 \end{bmatrix} $}
\end{aligned} 
\label{RNN5}   
\end{equation}

Thus, if a solution exists for Eq. \ref{RNN5}, it means that the following equations have a solution: 
\begin{equation}     
\scalebox{0.85}{$
\begin{aligned}         
&\sum_{l=1}^{m}w_y^{(l)} w_h^{(l)}(u_h^{(l)}+1)=0\\         &\sum_{l=1}^{m}w_y^{(l)} w_h^{(l)}(u_h^{(l)})^2=0 \\         &\dots \\         &\sum_{l=1}^{m}w_y^{(l)} w_h^{(l)}(u_h^{(l)})^{M-2}=0     \end{aligned}   $} 
\label{RNN6} 
\end{equation} 
To transform Eq. \ref{RNN6} into the matrix format as follows:
\begin{equation} 
\scalebox{0.75}{$\begin{bmatrix}   1+u_h^{(1)}& 1+u_h^{(2)} & \dots &1+u_h^{(m)} \\      (u_h^{(1)})^2 &  (u_h^{(2)})^2&  \dots & (u_h^{(m)})^2\\     \dots & \dots & \dots & \dots \\      (u_h^{(1)})^{M-2} &  (u_h^{(2)})^{M-2}&  \dots & (u_h^{(m)})^{M-2} \end{bmatrix}_{(M-2) \times m} \otimes    \begin{bmatrix} w_y^{(1)} w_h^{(1)} \\ w_y^{(2)} w_h^{(2)}\\ \dots \\ w_y^{(m)} w_h^{(m)}   
\end{bmatrix}=0    $}
\label{RNN7} 
\end{equation}

From Cramer’s Rule, if $(M-2)<m$, Eq. \ref{RNN7} must have non-zero solutions. In case of $(M-2)\ge m$, From the Vandermonde matrix:
\begin{equation}  \scalebox{0.75}{$
\mathbf{Rank}\begin{pmatrix} \begin{bmatrix}   1& 1 & \dots &1 \\      u_h^{(1)} &  u_h^{(2)}&  \dots & u_h^{(m)}\\     \dots & \dots & \dots & \dots \\      (u_h^{(1)})^{M-2} &  (u_h^{(2)})^{M-2}&  \dots & (u_h^{(m)})^{M-2} \end{bmatrix}_{M-1 \times m} \end{pmatrix} = m $}
\label{RNN8}  
\end{equation}
Eq. \ref{RNN8} is satisfied if and only if $u_h^{(i)} \ne u_h^{(j)}$ when $i \ne j$ \citep{klinger1967vandermonde}. Because of the independent m neurons assumption, $u_h^{(i)} \ne u_h^{(j)}, \quad \forall i \ne j$. Thus, the row vectors of the matrix term in Eq. \ref{RNN8} are linearly independent. Therefore:
\begin{equation}  
\scalebox{0.75}{$
\mathbf{Rank}
\begin{pmatrix} 
\begin{bmatrix}   1+u_h^{(1)}& 1+u_h^{(2)} & \dots &1+u_h^{(m)} \\      (u_h^{(1)})^2 &  (u_h^{(2)})^2&  \dots & (u_h^{(m)})^2\\     \dots & \dots & \dots & \dots \\      (u_h^{(1)})^{M-2} &  (u_h^{(2)})^{M-2}&  \dots & (u_h^{(m)})^{M-2} \end{bmatrix}_{(M-2) \times m}  \end{pmatrix} = m  $}
\label{RNN9}  
\end{equation} 
From Cramer's Rule, Eq. \ref{RNN7} only has the zero solution in the case of $(M-2) \ge m$. In theory, the space complexity of calculating the first-order derivative of $N$ sensors by vanilla RNN is $\mathcal{O}(3m) \approx \mathcal{O}(3M)$, while the time complexity equals $\mathcal{O}(mMN)) \approx \mathcal{O}(M^2N)$. To extend this conclusion, the space and time complexity will be at least $\mathcal{O}(3M(M-1)/2)\approx \mathcal{O}(M^2)$ and $\mathcal{O}(M^2-M)MN/2)\approx \mathcal{O}(M^3N)$, if we ascertain $(M-1)$-order derivative by stacking the vanilla RNN layers.  
\end{proof}

The preceding analysis assumes no missing data. However, in the presence of missing data, utilizing RNNs for derivative calculation may face challenges in accurately computing derivatives. This approach depends on the data's temporal features and is affected by the lack of sufficient data to capture this derivative feature. Additionally, prior knowledge of the physics law behind the data or assumptions about the underlying physics is required before devising the neural network architecture, which is impractical for unknown dynamic systems.

\begin{theorem}   
Given a single-layer neural network supplied with the first-order derivative of a time series of length $M$, it possesses the capability to compute combinations of derivatives of any order for the same $M$ length time series.
 \label{theorem2}
\end{theorem}    
\begin{proof}  For the convenient of the following discussion, we adopt $\mathbf{H}$ as the Block matrix. The first-order derivative vector $\mathbf{HX}$ is input to the single-layer neural network, ignoring the activation function. The operation can be defined as follows:  \begin{equation}    g(\mathbf{X}) = \mathbf{W}\mathbf{H}\mathbf{X}   \label{single}  \end{equation}   Where $\mathbf{W} \in \mathbb{R}^{M \times M}$ are the learnable parameters. Therefore, the proof is equivalent to $\forall \mathbf{\lambda}=[\lambda_{1},\lambda_{2},\cdots,\lambda_{M-1}]$, $\exists \mathbf{W}$  meets the following equation: 
\begin{equation}    
\begin{aligned}    
\mathbf{W} \mathbf{H} &= [\mathbf{H},\mathbf{H}^2,\cdots,\mathbf{H}^{M-1}]*[\lambda_{1},\lambda_{2},\cdots,\lambda_{M-1}]^T \\
&=\lambda_{1}\mathbf{H} + \lambda_{2}\mathbf{H}^2 + \dots + \lambda_{M-1}\mathbf{H}^{M-1}   
\end{aligned}
\label{Teoplitz} 
\end{equation}  
$Rank(\mathbf{H})=M$ and the product of two lower triangular matrices is also a lower triangular matrix \citep{strang2022introduction}. $Rank([\mathbf{H},\mathbf{H}^2,\cdots,\mathbf{H}^{M-1}])=M$, thus, $[\mathbf{H},\mathbf{H}^2,\cdots,\mathbf{H}^{M-1}]$ is linear independent. Moreover, as $Rank(\mathbf{H})=Rank([\mathbf{H},\mathbf{H},\mathbf{H}^2,\cdots,\mathbf{H}^{M-1}])=M$, for $\forall \mathbf{\lambda}$, Equation \ref{Teoplitz} must have a unique solution. 
\end{proof}

The theorem \ref{theorem2} indicates that we can using the single-layer neural network supplied with Teoplitz matrix  dynamically learns the physics law from the dataset, making it more adaptive to different datasets without prior assumptions about the underlying physics.

Instead of calculating the combination of High-order Toeplitz matrices in Eq. \ref{missmatrix} directly, a single-layer neural network can be used to learn the combination of different orders of Toeplitz matrices. In addition, calculating the combination of mixed $(M-1)$-order Toeplitz matrix by the learnable parameter matrix $\mathbf{W} \in \mathbb{R}^{M \times M}$, the time complexity can be reduced from $\mathcal{O}(M^{3(M-2)}N)$ (calculating the high order matrix directly) to $\mathcal{O}(M^3N)$, but the space complexity of will be increased from $\mathcal{O}(M-1)$ to $\mathcal{O}(M^2)$.

\subsection{Computational complexity} The time complexity of calculating the $\mathbf{L}^k$ and $\mathbf{H}^m$ directly can be represented as $\mathcal{O}(N^{3(k-1)}+M^{3(m-1)})$. Instead, the calculation of the $k$-hop order GCN neighborhood can be approximated by the Chebyshev polynomial to reduce the time complexity \citep{kipf2016semi}. In each propagation, the time for recursive sparse-dense matrix multiplication can be reduced to $\mathcal{O}(|\mathcal{E}|)$. However, calculating the $m$-order Toeplitz matrix has a high computational time complexity. Instead of using learning matrix $\mathbf{W} \in \mathbb{R}^{M \times M}$ in Eq. \ref{single}, $k_t$ dimension 1D convolutional filter is used to approximate the mix order Toeplitz matrix $\mathbf{WHX}$ as $conv(\mathbf{HX})$. Through this method, we can reduce the time complexity from $\mathcal{O}(M^3N)$ to $\mathcal{O}(k_tM^2N)$, while the space complexity can be reduced from $\mathcal{O}(M^2)$ to $\mathcal{O}(k_t)$ with $k_t \ll M$. Thus, the total time complexity of calculating the mix of high-order $\mathbf{L}$ and $\mathbf{H}$ can be reduced from $\mathcal{O}(N^{3(k-1)}M+M^{3(m-1)}N)$ to $\mathcal{O}(|\mathcal{E}|M +k_tM^2N)$ with $k_t \ll M$. 

\subsection{The future state prediction stage} If we obtain the ground truth values of the missing data in the training stage, then we can apply the error $|\mathbf{P}_t -\hat{\mathbf{P}}_t|$ for the backpropagation algorithm in the training stage. Unfortunately, the ground truth values of the missing data are only available for the evaluation stage since we can never obtain the missing values in real applications. Instead, we applied the imputation value $\hat{\mathbf{P}}_{t-M:t}$ and $\mathbf{U}_{t-M:t}$ to predict the future no missing data in the training step. To predict future values, we adopt the LSTM and Temporal Attention (TAtt) \citep{guo2019attention} in this stage. Hence, the predicted values and loss function can be formulated as follows: 
\begin{equation}  
\begin{aligned}  
&\hat{\mathbf{X}}_{t+1:t+M}=\mathrm{TAtt}(\mathrm{LSTM}(\hat{\mathbf{P}}_{t-M:t}+\mathbf{U}_{t-M:t}))\\  
&\mathcal{L}=l_{1loss}((\mathbf{X}_{t+1:t+M}-\hat{\mathbf{X}}_{t+1:t+M})\odot (1-\mathbf{M}_{t+1:t+M}))  
\end{aligned} 
\label{objectiveFunction}
\end{equation} 

\subsection{Explainability for HSPGNN}
Explainability in GNNs is essential to various applications. While utilizing the principles of physics aids in attaining a certain level of explainability regarding the imputation mechanism, it remains crucial to assess the significance of different nodes within the entire graph. This evaluation is pivotal for practical applications, as not all nodes exert the same influence. One promising approach to GNN explanation lies in instance-level explanations, which attempt to explain GNN by identifying important input features for its prediction. Assessing the impact of missing nodes serves as an effective means to glean insights into GNN operations. Analogous to the equilibrium forces in physical laws, analyzing the influence of each node across the entire graph by evaluating its absence is instrumental. Fidelity \cite{pope2019explainability} is an important metric of explainability in GNN, which follows the intuition that removing the truly important features will significantly decrease the model performance. In this paper, we evaluate fidelity by the missing impact $\mathbf{Y}$ of nodes, which can be defined as:
\begin{equation}
 \mathbf{Y}= \text{HSPGNN} (\mathbf{P})=\text{HSPGNN} (\mathbf{X}) -\text{HSPGNN} (\mathbf{U})
\label{missingEffect}
\end{equation}
However, it is difficult to calculate the missing impact since many datasets we obtained often contain originally missing data. Therefore, $\mathbf{X}$ is difficult to obtain in some situations. Instead,
to represent useful information about the missing values themselves, we assume that the missing value $\mathbf{P}$ and the corresponding missing impact $\mathbf{Y}$ are associated with distribution $\mathcal{D}$  with a latent variable $\mathbf{Z}$. Thus, our goal is to optimize the useful information expectation as follows:
\begin{equation}
\underset{\mathcal{D}_{\mathbf{Z}}}{\text{argmax}}\mathbb{E}_{\mathbf{Z}} \left [ p(\mathbf{Y}|\mathbf{P},\mathbf{Z}) \right ] 
\label{expectation}
\end{equation}
By Bayes' theorem, Eq. \ref{expectation} is equivalent to:
\begin{equation} 
\underset{\mathcal{D}_{\mathbf{Z}}}{\text{argmax}}\mathbb{E}_{\mathbf{Z}} [ \ln{p(\mathbf{Z}|\mathbf{P},\mathbf{Y}) - \ln{p(\mathbf{P},\mathbf{Z})}}+ \underbrace{\ln{p(\mathbf{P},\mathbf{Y})}}_{Evidence} ]  
\label{expectation1} 
\end{equation}
However, the posterior distribution $p(\mathbf{Z}|\mathbf{P},\mathbf{Y})$ is intractable. To make it feasible,  we approximate it by the observed value as $p(\mathbf{Z}|\mathbf{P},\mathbf{Y}) \approx p(\mathbf{Z}|\mathbf{P}_{obs},\mathbf{Y}_{obs})$, where $\mathbf{P}_{obs}$ and $\mathbf{Y}_{obs}$ are denoted as the observed missing value and the observed missing impact. 
 Also, we adopt Normalizing Flows (NF) to construct more accurate posterior distributions. NF \cite{rezende2015variational} provides a way of constructing flexible probability distributions by converting a simple distribution (e.g., uniform or Gaussian) into a complex one and can, in theory, approximate any complicated distribution. The mechanism can be formulated as:
\begin{equation}  
\ln{q_K(\mathbf{Z}_K)}=\ln{q_0(\mathbf{Z}_0)}- \sum_{k=1}^{K}\ln{\text{det} \left | \frac{\partial f_k}{\partial \mathbf{Z}_k} \right | },
\label{Norm}
\end{equation}
where $\text{det} \left | \frac{\partial f_k}{\partial \mathbf{Z}_k} \right |$ is the Jacobians of functions $f$.
Thus, from Eq. \ref{Norm}, ignoring the evidence term, we can formulate Eq. \ref{expectation1} as:
\begin{equation}
\begin{aligned}
\underset{\mathcal{D}_{\mathbf{Z}}}{\text{argmax}}\mathbb{E}_{\mathbf{Z}} \left [ p(\mathbf{P},\mathbf{Y}|\mathbf{Z}) \right ] \propto 
\underset{\mathcal{D}_{\mathbf{Z}_K}}{\text{argmax}}&\mathbb{E}_{\mathbf{Z}_0}  [ -\ln{p(\mathbf{P}_{obs},\mathbf{Z}_K)} \\
 + \ln{q_0(\mathbf{Z}_0)}& - \sum_{k=1}^{K}\ln{\text{det} \left | \frac{\partial f_k}{\partial \mathbf{Z}_k} \right | }  ] 
\end{aligned}
\label{FinalExpectation} 
\end{equation}

In this study, we adopt planar flow $f$ and initialization $\mathbf{Z}_0$ as:
\begin{equation} 
\begin{aligned}
f(\mathbf{Z})= &\mathbf{Z} + \mathbf{U}_0 tanh( \mathbf{U}_1^T \mathbf{Z}+\mathbf{b})\\
 &\mathbf{Z}_0\sim \mathcal{N}(0, \mathbf{I}),
\end{aligned}
\end{equation}
where $\mathbf{U}_0, \mathbf{U}_1,\mathbf{b}$ are the learnable parameters and $tanh$ is the activation function. Through Eq. \ref{FinalExpectation}, we can optimize the parameters and the distribution $\mathcal{D}_\mathbf{Z}$, which can indicate the impact effect of each node with originally missing datasets.

\section{Experiments}
\textbf{Datasets and experimental setting:} To evaluate the effectiveness of HSPGNN, four datasets from different domains with different graph structures are collected for the experiments: (1) Air Quality (AQI) \citep{yi2016st}: This dataset is a benchmark for the time series imputation from the Urban Computing project of 437 air quality monitoring stations over 43 cities in China. AQI-36 is the reduced set of 36 stations from AQI. (2) PeMS-BAY \citep{li2017diffusion}: PeMS-BAY is a famous traffic dataset in spatio-temporal problems. (3) Electricity \citep{misc_electricityloaddiagrams20112014_321}: This is a public dataset that is released by UCI. For AQI and AQI-36 datasets, we applied the same evaluation protocol of the baseline \citet{yi2016st} for comparison. For PeMS-BAY and Electricity datasets, we emulate block missing (randomly drop 5\% of the available data for each sensor and a failure of probability pfailure = 0.15\% with a uniform duration between 12 and 48-time steps) and point missing mode (randomly drop 25\% available data) settings as the previous work of \citet{andrea2021filling}. Notably, in both cases, the model does not have access to the ground-truth values of the missing data used for the final evaluation. Also, we denote HSPGNN-L as adding one more LSTM and Temporal Attention layer than the HSPGNN model. To better consider the performance of our method, we only consider the out-of-sample scenarios in all datasets, which means the training and evaluation sequences are disjointed. To facilitate comparison with baseline methods, we adopt the same metrics standard as used in previous baseline studies as GRIN \citep{andrea2021filling}. All experiments are conducted on Pytorch 1.21.1 based on a Linux server (CPU: Intel(R) Core(TM) i7-1800H @2.30GH, GPU: NVIDIA GeForce RTX 3080) with 32G memory. Adam Optimization is utilized at an initial learning rate of 0.0005 with decay rates of 0.92 every epoch. The batch size is 8 in the PeMS-BAY dataset, while the others are 16. $M$ and $K_t$ are set to 60 and 3, and residual connections are applied to reduce the gradient vanishing problem. $L1$ is adopted as the loss function. 60 samples are used to train our model. Also, for all datasets, 16\% of the training set is used for the validation set. As for the node importance experiment, we set the $K$ of norm flow as 480, the batch size is 100 and the learning rate is 0.0001.

\textbf{Results:} Tab. \ref{tab1} show the results of experiments on the four datasets. The results indicate that our model can achieve the best performance except for 25\% points missing in PeMS-BAY. As for the traditional spatio-temporal dataset as PeMS-BAY, deep learning methods outperform other methods for their excellent capability of automatically extracting relevant features from raw data. Among all deep learning methods, GRIN and HSPGNN achieve the best performance since they consider both spatio-temporal features rather than those only representing the temporal features. HSPGNN assumes the linear relationship to approximate the real dynamic system. However, we believe that real scenarios are more complex than those represented by a linear model. In situations where the missing rate is not high or there are simple missing patterns, data-driven models like GRIN may effectively leverage the available data to capture nonlinear relationships. Consequently, HSPGNN does not achieve the best performance when 25\% of the data points are missing. From Fig. \ref{MAE1}, as the randomly missing rate increases, the MAE of all models increases. Clearly, multivariate time series imputation is still a challenging problem. But our physics-incorporated model can obtain better robustness with the least change and the best performance over 80\% missing rate as well as block missing scenarios since it takes both the deep learning model and physics model into consideration. As for the real missing dataset AQI and AQI-36 that contain real various missing patterns, our model obtains the best result in Tab. \ref{tab1}, which can fully demonstrate the superiority of our model in complex missing data situations. 

As for the Electricity dataset, traditional methods such as MF have a solid theory foundation, but the intrinsic relationships of nodes are Non-grid structures, which makes it difficult to represent the latent graph structure of time series data. Hence, the imputation performance on time series data is unsatisfactory. GRIN is constrained to learn how to perform imputation by looking at static neighboring nodes, which makes it hard to fully exploit the dynamic features of graphs. From Fig. \ref{MAE2}, our model achieves the best performance at different missing rates since it can capture the dynamic relationship through spatial attention and represent the dynamic physics-incorporated system effectively.

\begin{table*}[!hbt]  
\small
\centering 
\caption{The imputation results of different methods on different datasets.} 
\scalebox{0.6}{
\begin{tabular}{lllllllllllll}
\hline
\multicolumn{1}{c|}{}       & \multicolumn{4}{c}{PeMS\_BAY}                                                                         & \multicolumn{4}{c|}{Electricity}                                                                      & \multicolumn{2}{c}{\multirow{2}{*}{AQI}}          & \multicolumn{2}{c}{\multirow{2}{*}{AQI-36}}       \\ \cline{2-9}
\multicolumn{1}{c|}{Models} & \multicolumn{2}{c}{Block missing}                 & \multicolumn{2}{c}{Point missing}                 & \multicolumn{2}{c}{Block missing}                 & \multicolumn{2}{c|}{Point missing}                & \multicolumn{2}{c}{}                              & \multicolumn{2}{c}{}                              \\ \cline{2-13} 
\multicolumn{1}{c|}{}       & \multicolumn{1}{c}{MAE} & \multicolumn{1}{c}{MSE} & \multicolumn{1}{c}{MAE} & \multicolumn{1}{c}{MSE} & \multicolumn{1}{c}{MAE} & \multicolumn{1}{c}{MSE} & \multicolumn{1}{c}{MAE} & \multicolumn{1}{c}{MSE} & \multicolumn{1}{c}{MAE} & \multicolumn{1}{c}{MSE} & \multicolumn{1}{c}{MAE} & \multicolumn{1}{c}{MSE} \\ \hline
Mean                        & 5.46±0.00               & 87.56±0.00              & 5.42±0.00               & 86.59±0.00              & 1.73±0.00               & 5.15±0.00               & 1.74±0.00               & 5.18±0.00               & 39.60±0.00              & 3231.04±0.00            & 53.48±0.00              & 4578.08±0.00            \\
KNN \citep{cover1967nearest}                          & 4.30±0.00               & 49.90±0.00              & 4.30±0.00               & 49.80±0.00              & 0.62±0.00               & 1.45±0.00               & 0.63±0.00               & 1.54±0.00               & 34.10±0.00              & 3471.14±0.00            & 30.21±0.00              & 2892.31±0.00            \\
MF \citep{cichocki2009fast}                           & 3.28±0.01               & 50.14±0.13              & 3.29±0.01               & 51.39±0.64              & 0.22±0.01               & 0.13±0.01               & 0.23±0.01               & 0.14±0.01               & 26.74±0.24              & 2021.44±27.98           & 30.54±0.26              & 2763.06±63.35           \\
MICE \citep{white2011multiple}                         & 2.94±0.02               & 28.28±0.37              & 3.09±0.02               & 31.43±0.41              & 0.33±0.02               & 0.39±0.03               & 0.51±0.01               & 1.06±0.02               & 26.98±0.10              & 1930.92±10.08           & 30.37±0.09              & 2594.06±7.17            \\
MPGRU \citep{li2017diffusion}                       & 1.59±0.00               & 14.19±0.11              & 1.11±0.00               & 7.59±0.02               & 0.32±0.02               & 0.28±0.04               & 0.33±0.02               & 0.32±0.01               & 18.76±0.11              & 1194.35±15.23           & 16.79±0.52              & 1103.04±106.83          \\
Transformer \citep{vaswani2017advances}                  & 1.70±0.02               & 10.41±0.05              & 0.74±0.00               & 2.23±0.00               & 0.92±0.01               & 1.96±0.01               & 0.84±0.02               & 1.74±0.02               & 18.13±0.21              & 1004.35±18.36           & 12.00±0.60              & 519.46±54.34            \\
BRITS \citep{cao2018brits}                       & 1.70±0.01               & 10.50±0.07              & 1.47±0.00               & 7.94±0.03               & 1.02±0.01               & 1.04±0.01               & 0.93±0.01               & 1.98±0.02               & 20.21±0.22              & 1157.89±25.66           & 14.50±0.35              & 662.36±65.16            \\
GRIN \citep{andrea2021filling}                        & 1.14±0.01               & 6.60±0.10               & \textbf{0.67±0.00}      & \textbf{1.55±0.01}      & 0.60±0.00               & 1.32±0.02               & 0.56±0.00               & 1.24±0.02               & 14.73±0.15              & 775.91±28.49            & 12.08±0.47              & 523.14±57.17            \\
M$^2$DMTF \citep{fan2021multi}                   & 2.49±0.02               & 19.76±0.12              & 2.51±0.02               & 19.89±0.42              & 0.36±0.01               & 0.75±0.01               & 0.33±0.01               & 0.64±0.01               & 28.91±0.25              & 2150.00±28.25           & 14.61±0.49              & 737.41±45.34            \\
SAITS \citep{du2023saits}                       & 1.56±0.01               & 9.67±0.04               & 1.40±0.03               & 9.85±0.07               & 0.87±0.00               & 1.82±0.01               & 0.77±0.00               & 1.43±0.01               & 21.32±0.15              & 1234.24±53.24           & 18.13±0.35              & 1223.74±123.36          \\
Nhits \citep{challu2023nhits}                       & 1.52±0.01               & 11.22±0.11              & 1.89 ±0.01              & 16.15±0.13              & 0.42±0.01               & 0.11±0.01               & 0.37±0.01               & 0.29±0.01               & 24.44±0.25              & 1836.87±28.48           & 24.35±0.15              & 1398.794±24.43          \\
TDM\citep{zhao2023transformed}                         & 2.70±0.01               & 21.89±0.32              & 2.28±0.01               & 14.14±0.41              & 0.73±0.01               & 1.51±0.01               & 0.54±0.01               & 0.64±0.01               & 20.76±0.46              & 881.06±25.31            & 31.94±0.53              & 1644.66±166.34 
\\ \hline
\textbf{HSPGNN-L-ours}      & \textbf{1.10±0.02}      & \textbf{5.95±0.20}      & 0.78±0.02               & 2.35±0.20               & 0.18±0.01               & 0.20±0.01               & 0.11±0.01               & 0.03±0.01               & 13.30±0.20              & 598.91±26.85            & 11.25±0.23              & 439.72±53.35            \\
\textbf{HSPGNN-ours}        & 1.20±0.03               & 7.26±0.60               & 0.79±0.03               & 2.15±0.30               & \textbf{0.15±0.01}      & \textbf{0.10±0.01}      & \textbf{0.10±0.01}      & \textbf{0.02±0.01}      & \textbf{12.85±0.12}     & \textbf{574.09±23.78}   & \textbf{11.19±0.20}     & \textbf{438.82±51.43}   \\ \hline
\label{tab1} 
\end{tabular}}
\end{table*}

\begin{figure}[!hbt] 
\subfloat[PeMS-BAY.]{ 
\includegraphics[width=0.49\columnwidth]{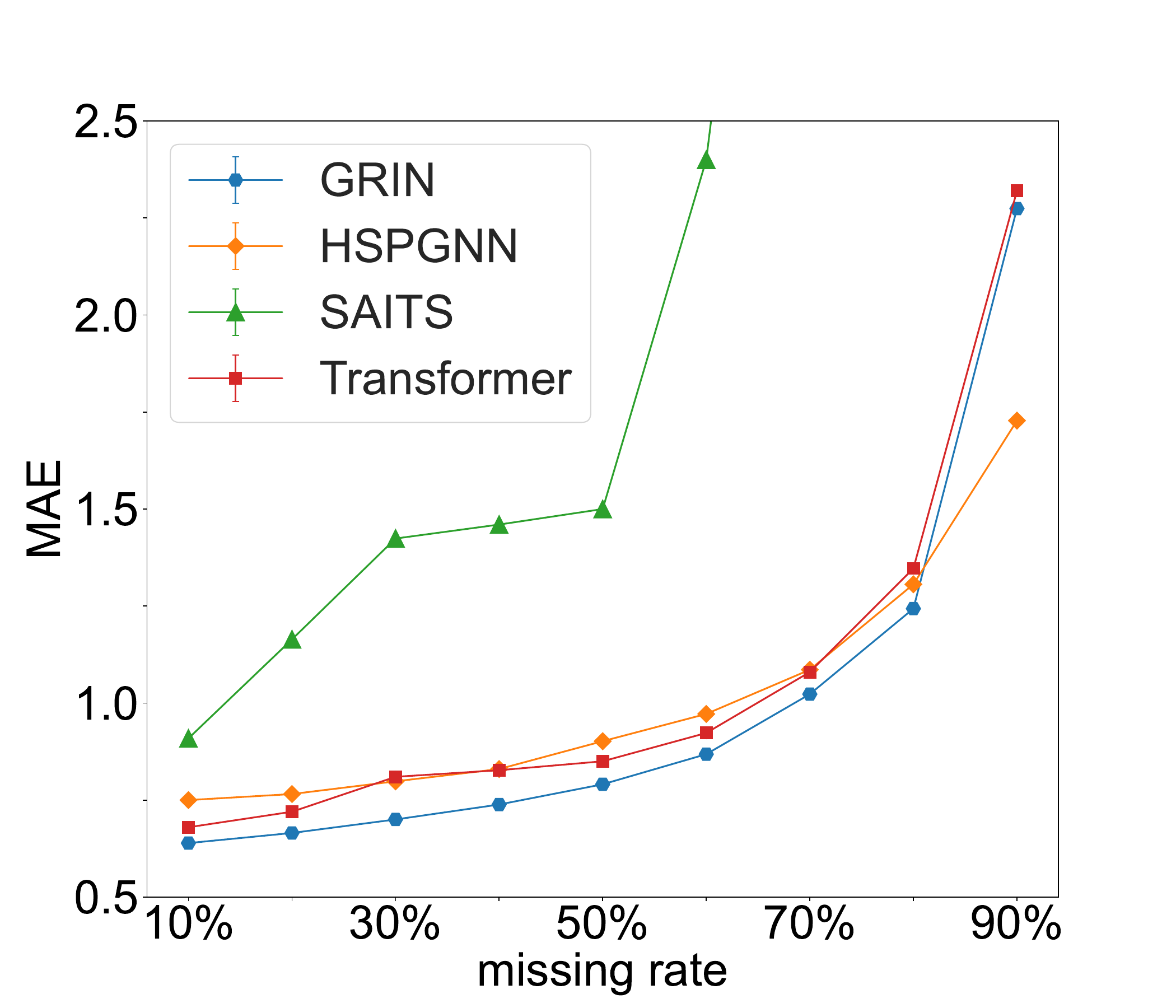} 
\label{MAE1} 
} 
\subfloat[Electricity.]{ 
\includegraphics[width=0.49\columnwidth]{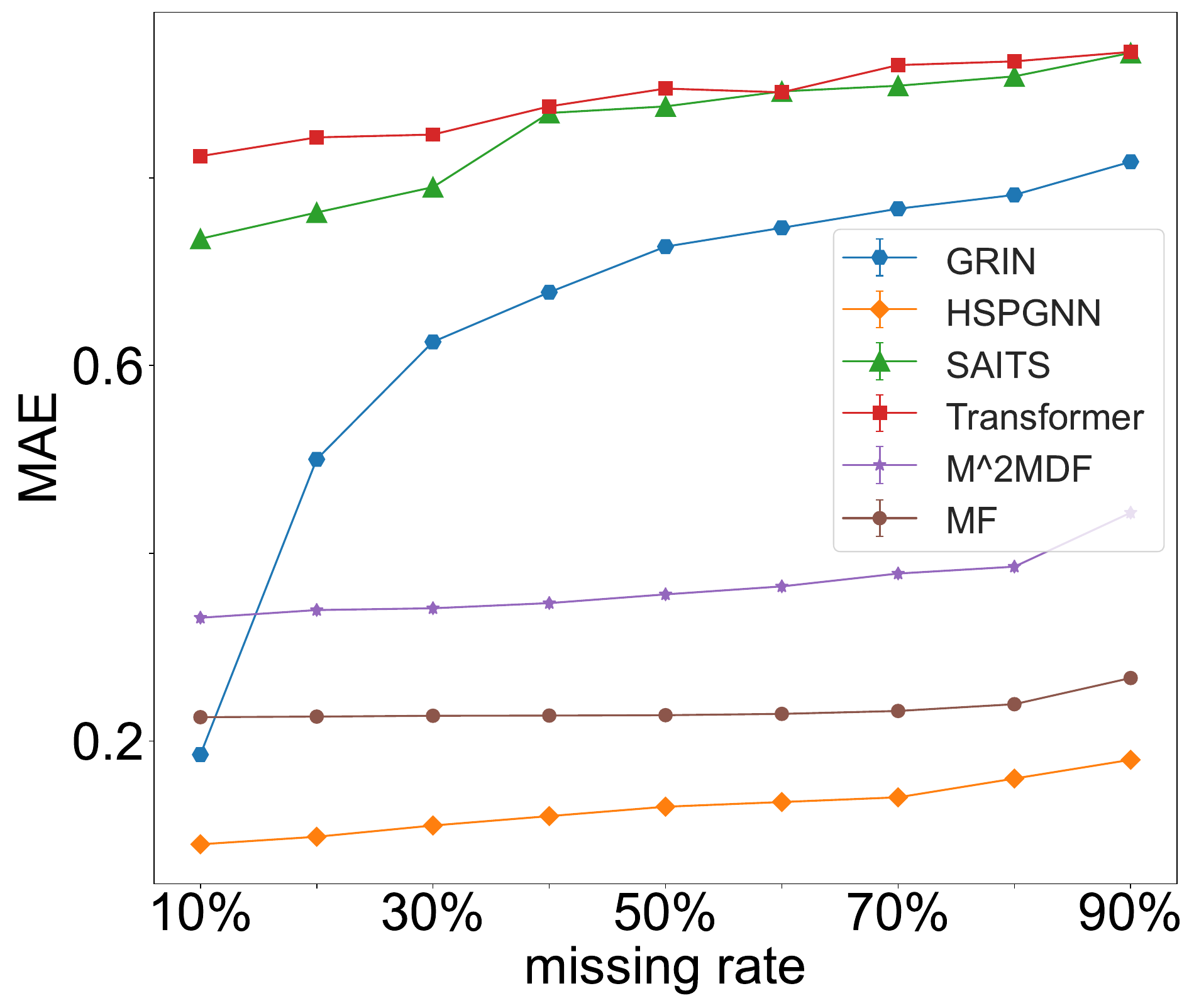}  
\label{MAE2} 
} 
\caption{MAE of different models at different missing rates on PeMS-BAY and Electricity datasets.}
\end{figure} 

\textbf{The impact of K-hop combination :} To explore the impact of different orders of neighbor nodes, the K-hop combination experiments are conducted. The results are shown in Tab. \ref{khop}. From the table, only considering one order of the physics equation does not lead to the best result. In real applications, the interaction between the neighbor nodes is not limited to the first-order neighbor but also the multi-hop neighbors. Therefore, the data naturally obey multi-physical laws with a mixture of different orders. For example, only considering such as the LWR model in the PeMS-BAY dataset or the wave propagation model in the AQI dataset makes it easy to ignore the effect of other order neighbors and decrease the performance of the model. Interestingly, the combination numbers of the best performance on different datasets are different. The result shows that combining the first and second hop in the large graph, such as in AQI and PeMS datasets, achieves the best performance, while the small graph in AQI-36 needs 3-hop to obtain the best performance. Thus, the results fully prove the existence of multi-hop physical correlations in the time series data.

\textbf{Time complexity result:} 
Experiments are conducted to compare the computation time of different methods to calculate the derivative of time in the training stage. The experiments are conducted based on the experiment setting, Also, we use the Vanilla RNN to calculate the $M-1$ derivative of time. The results shown in Tab. \ref{running_time} Verifies the effectiveness of our model.

\begin{table}[!hbt] 
\caption{The result of K hop.}
\centering 
\scalebox{0.85}{
\begin{tabular}{llll} \hline \multicolumn{1}{l|}{\multirow{2}{*}{Order of Neighbors}} & \multicolumn{1}{c}{AQI} & \multicolumn{1}{c}{AQI-36} & \multicolumn{1}{c}{PEMS\_BAY} \\ \cline{2-4}  \multicolumn{1}{l|}{}                                             & \multicolumn{1}{c}{MAE} & \multicolumn{1}{c}{MAE}    & \multicolumn{1}{c}{MAE}       \\ \hline K=1                                                      & 13.22±0.11                       & 11.57±0.20                         & 1.30±0.03                              \\ K=2                                                      & 13.17±0.12                       & 11.63±0.23                          & 1.27±0.03                              \\ K=3                                                      & 13.13±0.11                       & 11.71±0.21                          & 1.24±0.03                              \\ K=4                                                      & 13.15±0.10                       & 11.80±0.30                          & 1.25±0.02                              \\ K={[}1,2{]}                                              & \textbf{12.85±0.12}              & 11.27±0.23                          & \textbf{1.10±0.02}                     \\ K={[}1,2,3{]}                                            & 13.07±0.13                       & \textbf{11.19±0.20}                 & 1.13±0.02                              \\ K={[}1,2,3,4{]}                                          & 13.02±0.14                       & 11.20±0.21                          & 1.17±0.02                              \\ \hline 
\label{khop} 
\end{tabular}}
\end{table}

\begin{table}[!hbt] 
\caption{The result of running time experiments.}
\centering 
\scalebox{0.80}{
\begin{tabular}{llll}
\hline
Dataset(ms/epoch)                                       & Vanilla RNN & Calculate directly & Our method         \\ \hline
AQI                                                     & 32.42±15.64 & 314.21±75.64       & \textbf{3.99±2.65} \\
PeMS-BAY                                                & 75.54±32.13 & 482.54±95.64       & \textbf{6.94±2.43} \\
Electricity                                             & 69.43±28.33 & 448.06±35.43       & \textbf{6.45±2.31} \\ \hline
\label{running_time} 
\end{tabular}}
\end{table}

\textbf{Explainability of HSPGNN:} 
As for optical flow-based moving object detection in the computer vision domain, provided that in the neighborhood of pixel, change of brightness intensity $I(x,y,t)$ does not happen motion field and ignores the high order term, we can use the following expression the moving object can be formulated as follows \citep{aslani2013optical}: 
\begin{equation} 
V_t = -\frac{\partial I(x,y,t)}{\partial t}/\triangledown I(x,y,t).
\end{equation}
Where $\triangledown I(x,y,t)$ is so-called spatial gradient of brightness intensity and $V_t$ is the optical flow (velocity vector) of the image pixel. Optical flow can also be defined as the distribution of apparent velocities of movement of brightness pattern in an image. Similarly, ignoring the higher-order term and external source, we can obtain a similar graph-like optical flow from Eq. \ref{matrix} as follows: 
\begin{equation}  
\scalebox{0.9}{$
\mathbf{V}^{ij}_t=-\frac{ \mathbf{X}^i_t - \mathbf{X}^i_{t-1}}{\mathbf{X}_t^i - \mathbf{X}_{t}^j}=-\frac{\lambda_1 }{\mathbf{\Theta}_1 \mathbf{L}_t^{ij}}. $} 
\end{equation} 
where $\mathbf{V}^{ij}_t$ represent the graph-like optical flow from $i$-th node to $j$-th node at time $t$, while $\lambda_1$, $\mathbf{\Theta_1}$ and $\mathbf{\mathbf{L}_t^{ij}}$ are the learnable parameters of the neural network and Laplacian matrix. Thus, we can obtain the transfer speeds between the nodes in graph. We draw the graph-like optical flow in AQI and PeMS-BAY datasets. As shown in Fig. \ref{optical2}, the red circle is the center node, while the blue, yellow, dark green, and cyan circles correspond to their first to fourth neighbors, respectively. The thickness of the lines reflects the optical flow in graphs. The thicker the lines are, the higher the speeds of the neighbor node transfer. As for the AQI dataset, we chose the 219-th node as the center node for three different months. At the same time step, the optical flows between some nodes are different and even disappear, although they are of the same order neighbors. At different time steps, some disappearing lines will emerge again. Also, the optical flows between some neighbor nodes are stable with little change over time, while others change dramatically. As for the PeMS-BAY dataset, we chose the 20th node as the center node with more neighbor nodes at three adjacent weeks. Unsurprisingly, We can observe the same phenomenon as the AQI dataset, which is useful in analyzing the air pollutants or traffic flow transition except only filling the missing data. Hence, our model can discover the dynamic relationship by incorporating the physics layer, which not only makes our model more robust to different data missing patterns but also more explainable to the internal physical relationships hidden in the datasets.

\begin{figure}[!hbt]  
\centering 
\subfloat[At March.]{  
\centering 	 
\includegraphics[width=0.33\columnwidth]{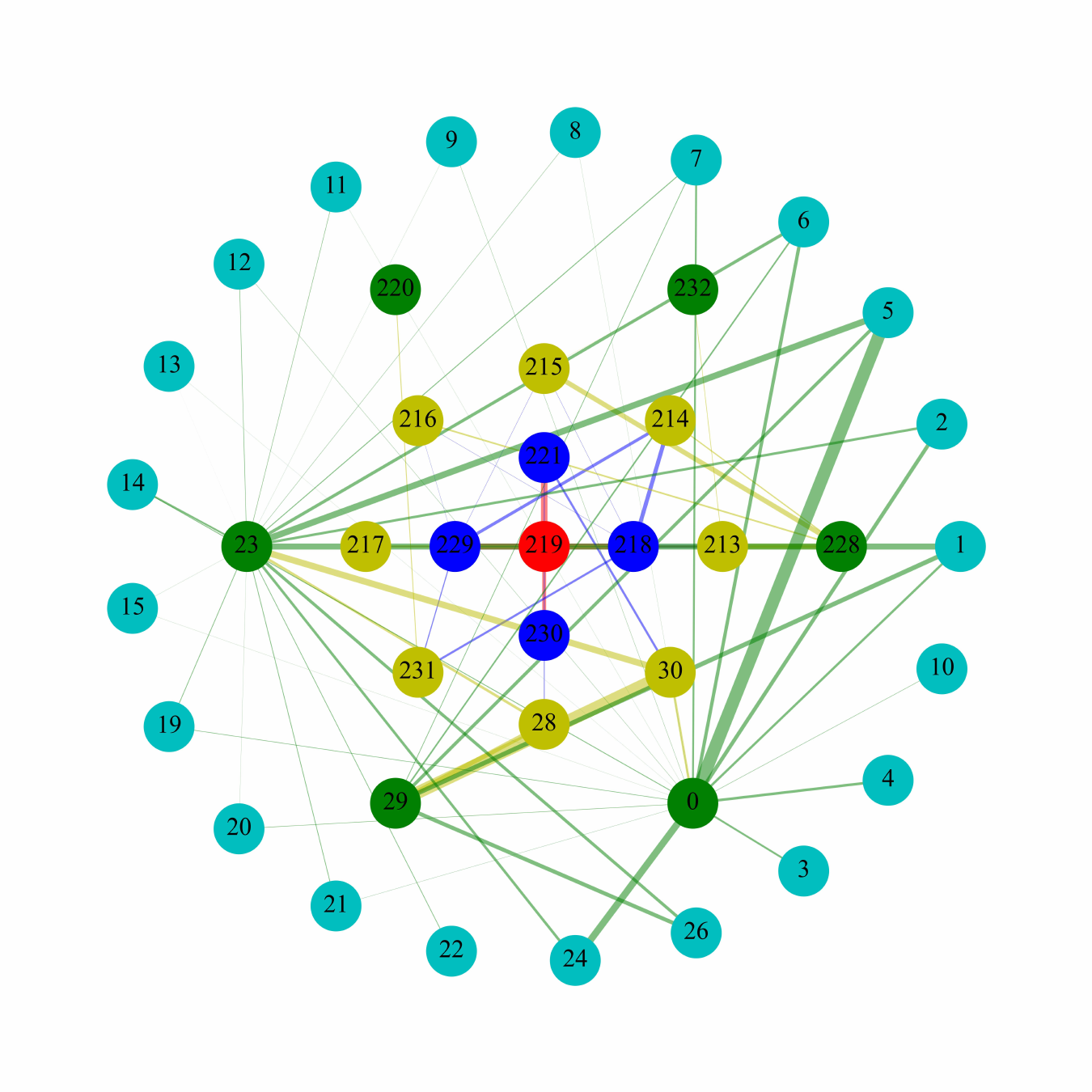} 	
\label{baseline_curve_d} }   
\subfloat[At June.]{ 
\centering 
\includegraphics[width=0.33\columnwidth]{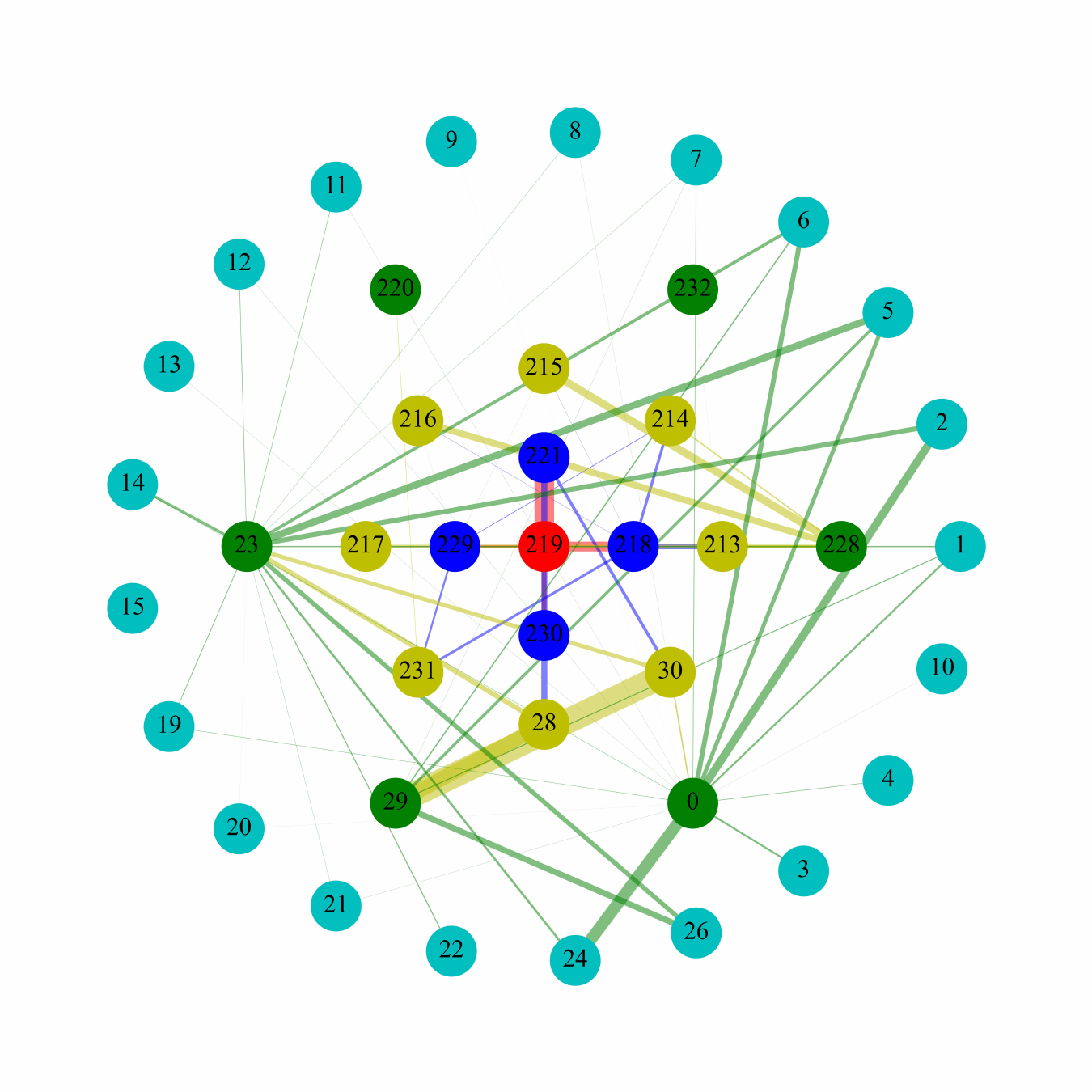} 
\label{baseline_curve_e} } 
\subfloat[At September.]{ 
\centering  
\includegraphics[width=0.33\columnwidth]{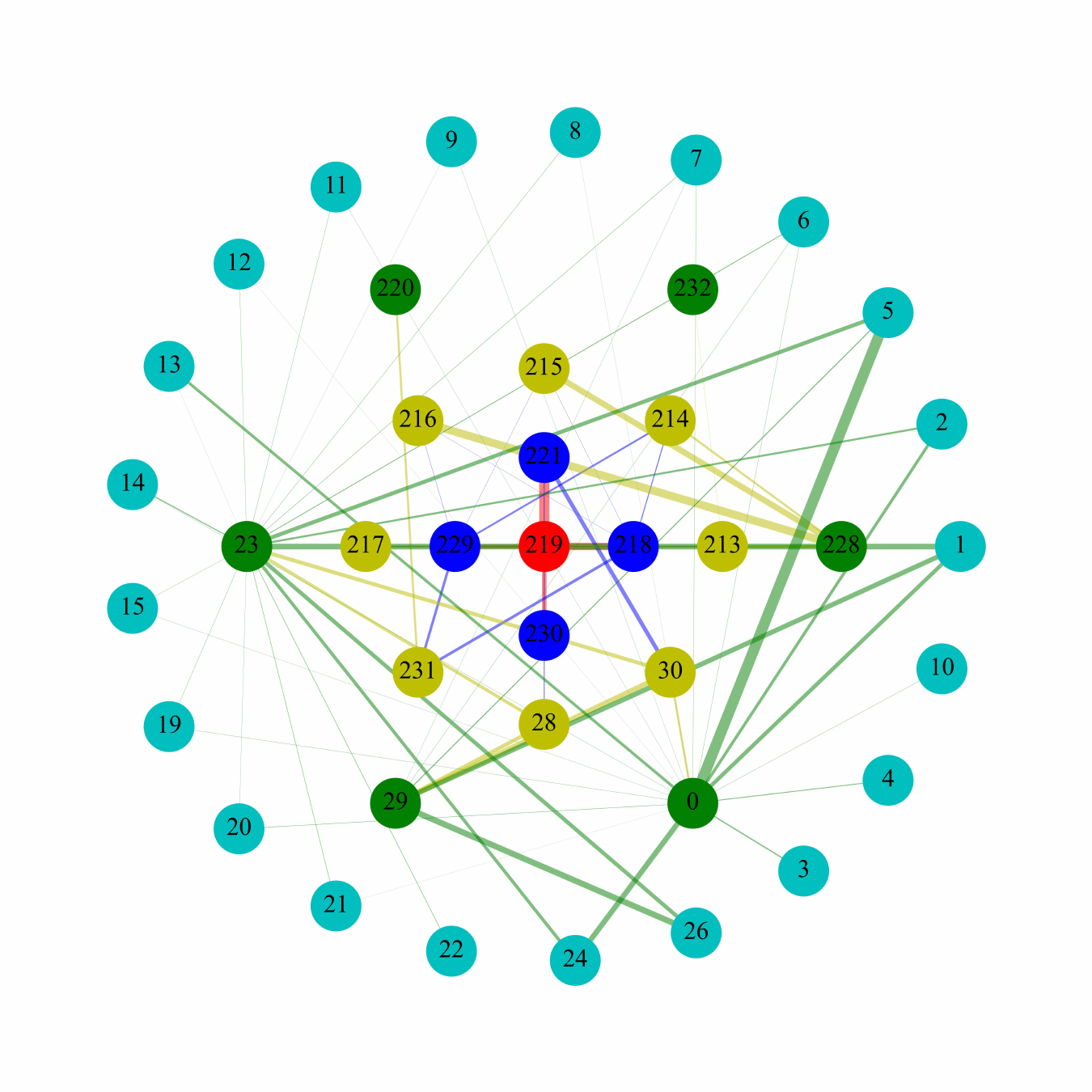} 
\label{baseline_curve_f} }

\centering 
\subfloat[At week 1.]{ 
\centering 	 
\includegraphics[width=0.33\columnwidth]{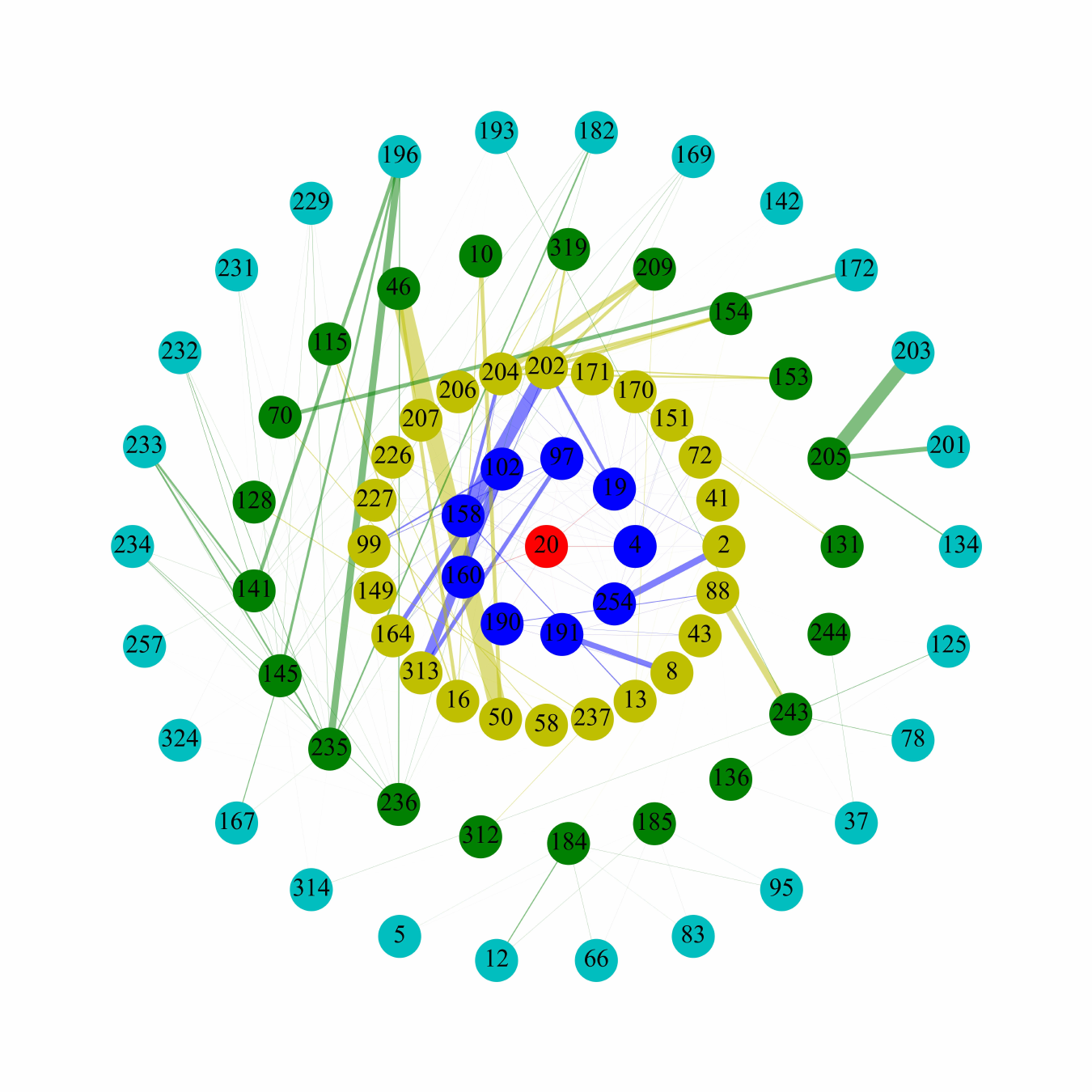} 	
\label{baseline_curve_d} }   
\subfloat[At week 2.]{ 
\centering  \includegraphics[width=0.33\columnwidth]{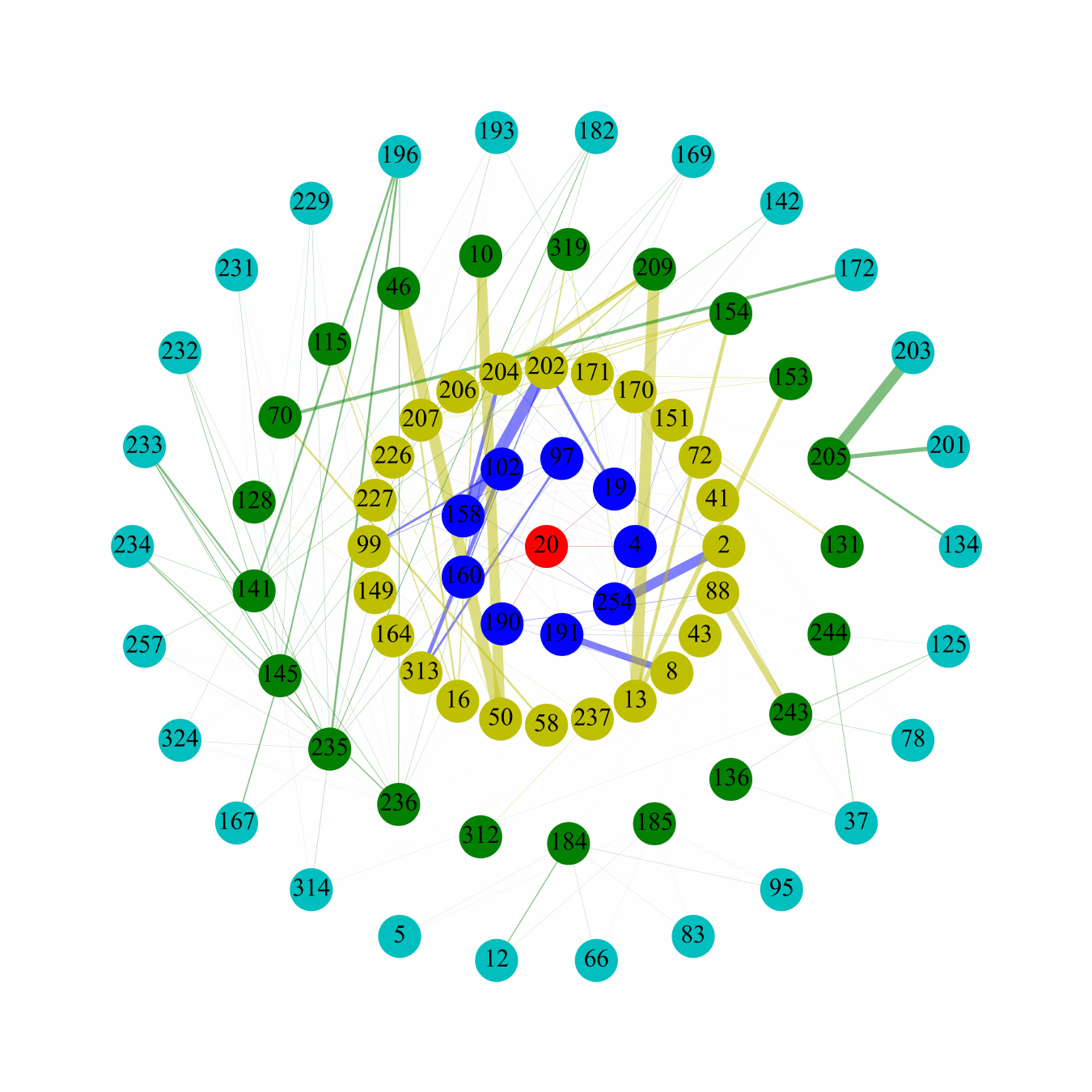} 
\label{baseline_curve_e} }  
\subfloat[At week 3.]{ 
\centering  
\includegraphics[width=0.33\columnwidth]{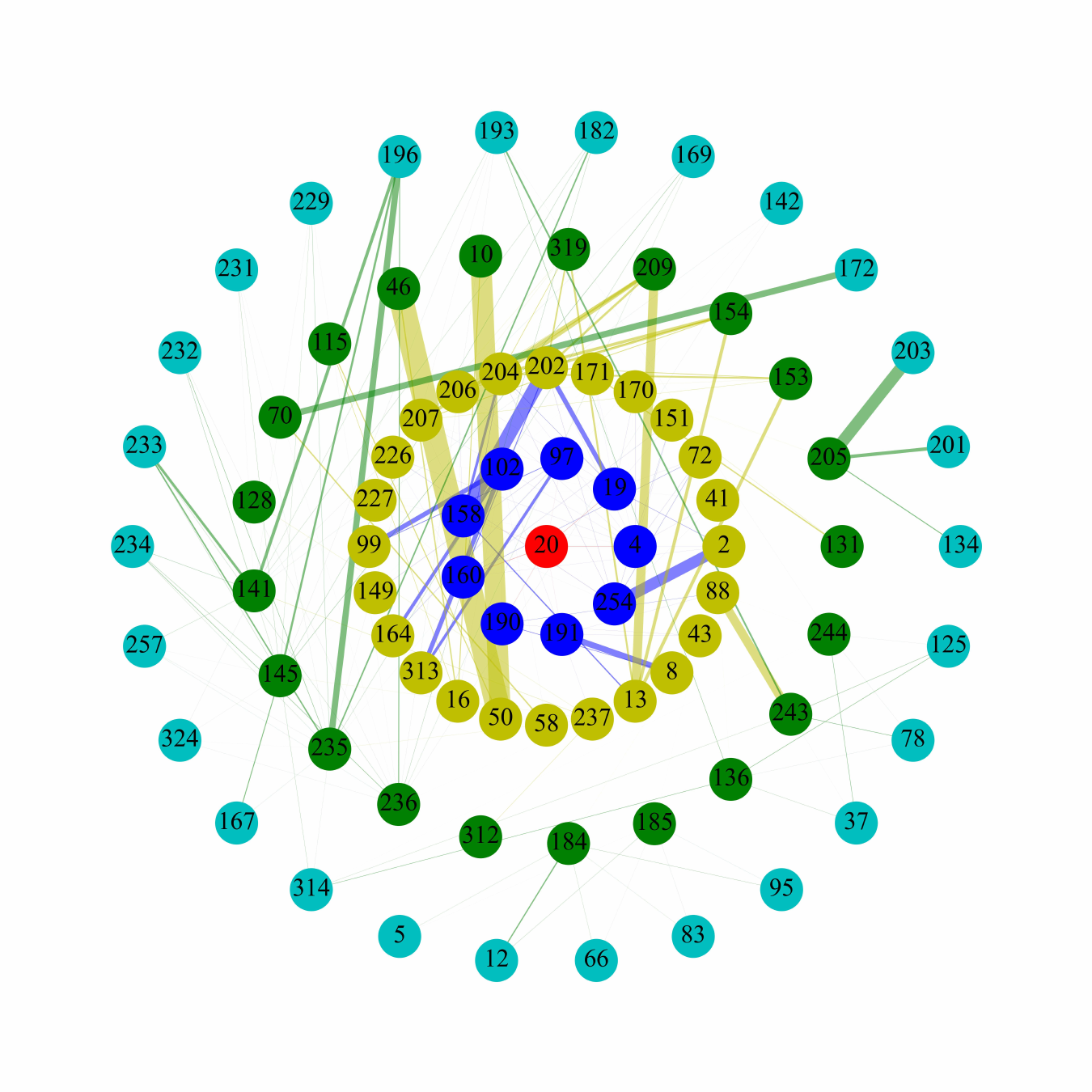}  
\label{baseline_curve_f} } 
\caption{The graph-like optical flow at different month and weeks on AQI and PeMS-BAY datasets.} 
\label{optical2} 
\end{figure}

As for the Electricity dataset, we choose 50 clients and draw the dynamic graph in Fig. \ref{dynamic_graph1}, \ref{dynamic_graph2} and \ref{dynamic_graph3} since there is no apparent original physical (e.g., geographic) proximity. From the diagram, we observe that the adjacent matrices at the adjacent time steps change little. It is reasonable to conclude that abnormal alterations in electricity consumption habits within a short time are infrequent among the majority of clients. However, while the long time interval appears to have a large change, from Fig. \ref{dynamic_graph4}, \ref{dynamic_graph5}, and \ref{dynamic_graph6}, different clients over different months show significant changes, and the adjacent matrices are changing dynamically and adaptively over time.  

\begin{figure}[!hbt]  
\centering 
\subfloat[At time step 1.]{
\centering 	 
\includegraphics[width=0.33\columnwidth]{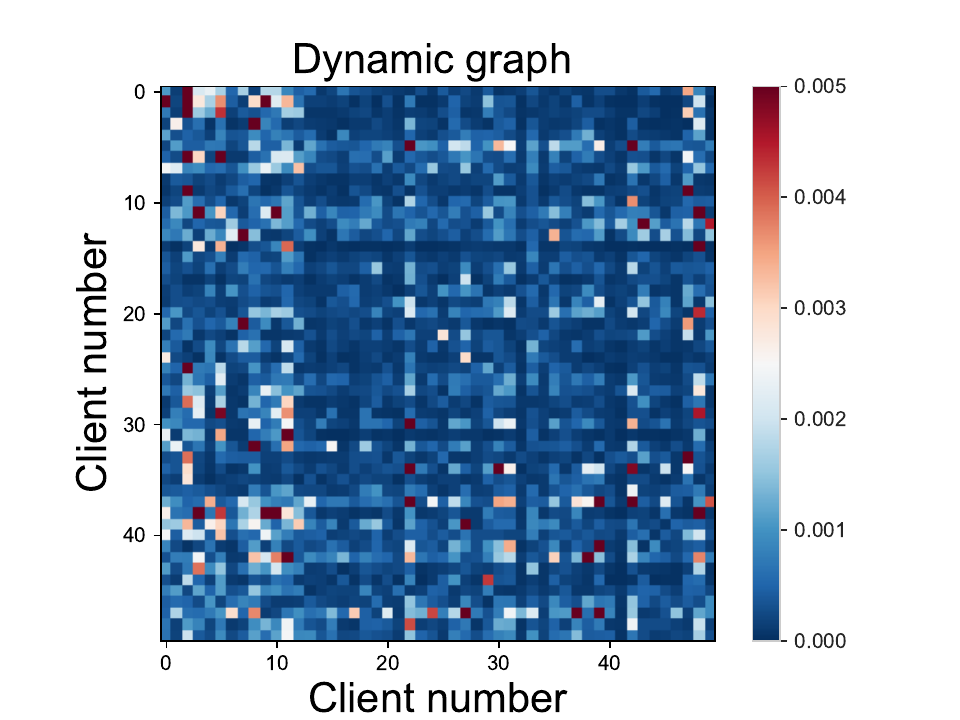} 	
\label{dynamic_graph1} }   
\subfloat[At time step 2.]{  
\centering  
\includegraphics[width=0.33\columnwidth]{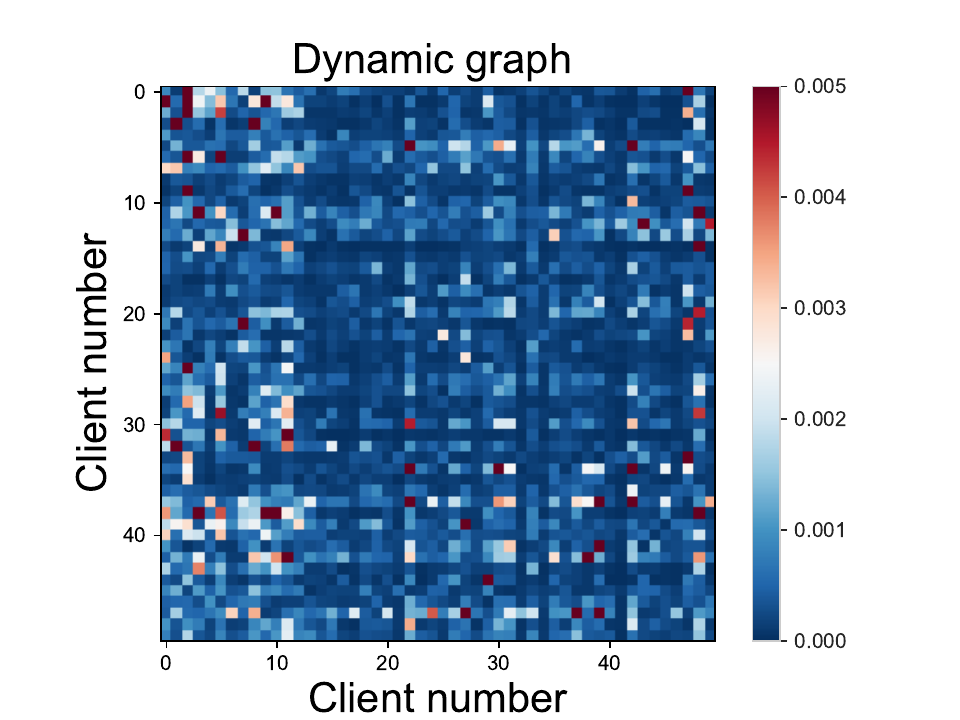}  
\label{dynamic_graph2} }   
\subfloat[At time step 3.]{  
\centering   
\includegraphics[width=0.33\columnwidth]{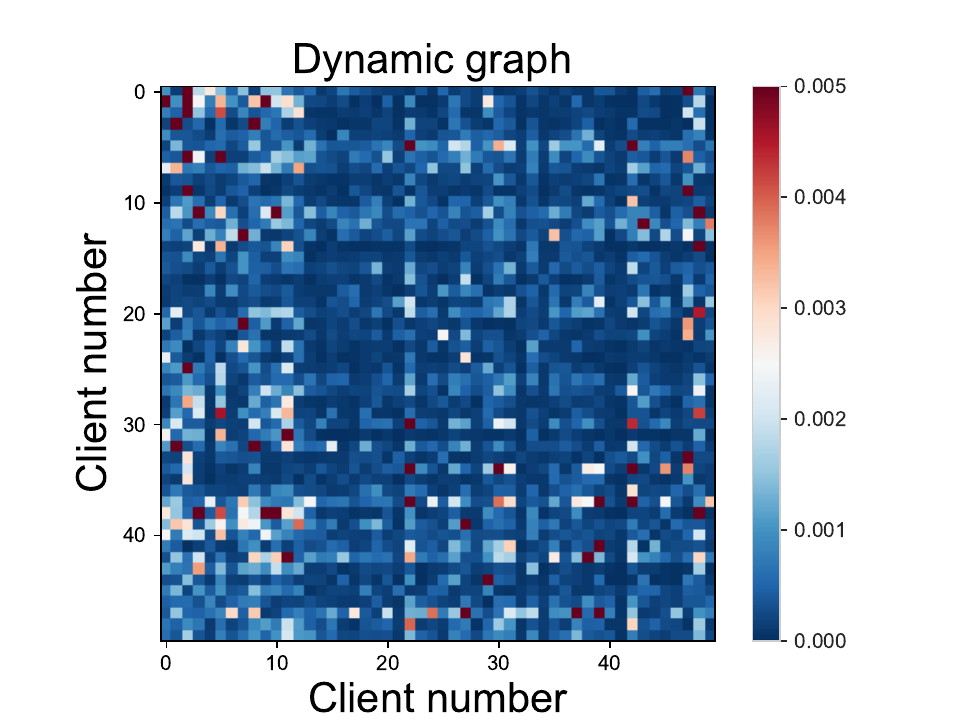}   
\label{dynamic_graph3} } 
\\

\centering  
\subfloat[At Jan.]{ 
\centering 	 
\includegraphics[width=0.33\columnwidth]{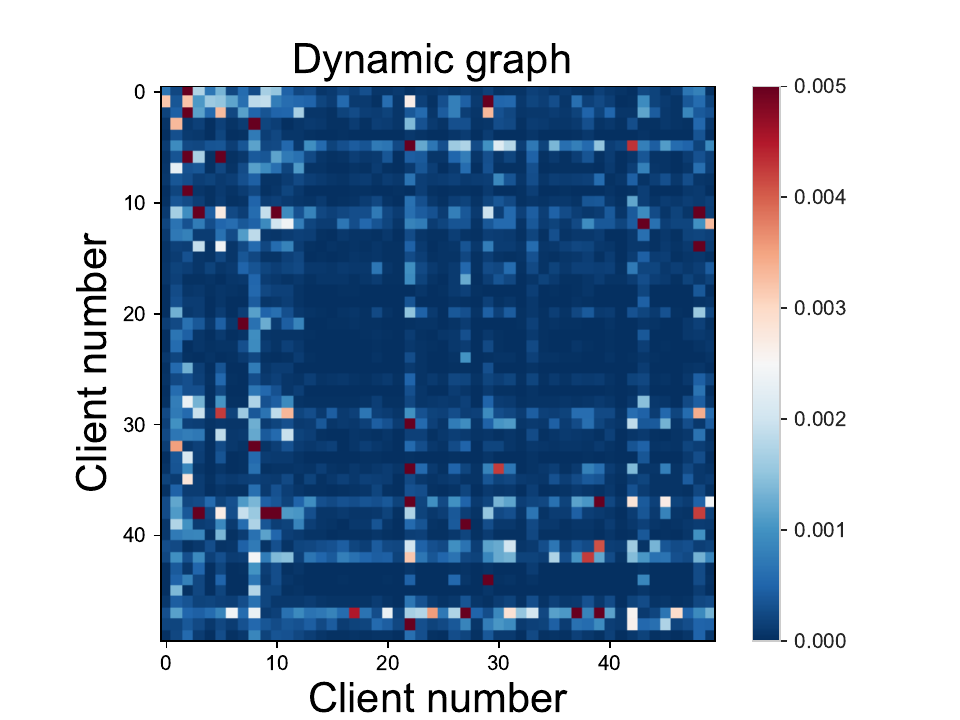} 	 
\label{dynamic_graph4} } 
\subfloat[At Feb.]{  
\centering   
\includegraphics[width=0.33\columnwidth]{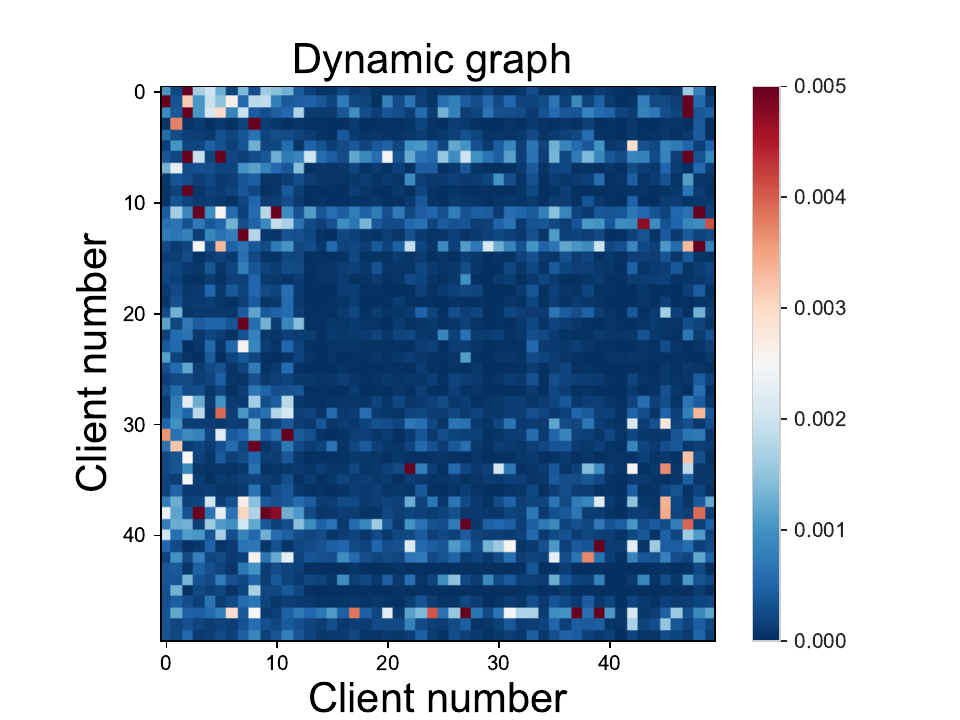}  
\label{dynamic_graph5} }   
\subfloat[At Mar.]{ 
\centering   
\includegraphics[width=0.33\columnwidth]{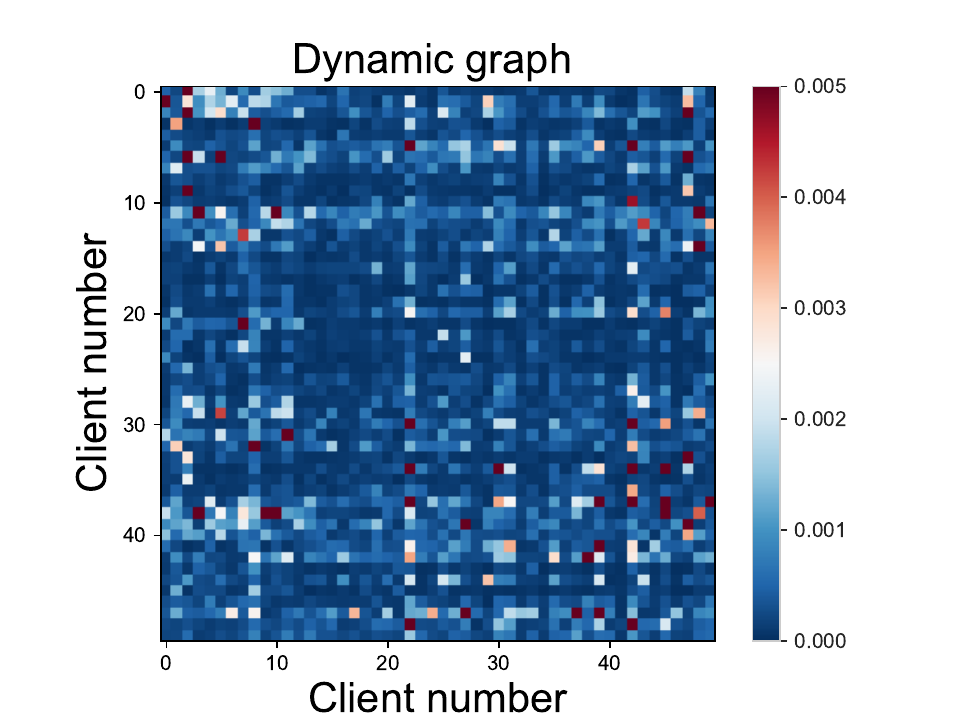}  
\label{dynamic_graph6} } 
\caption{Dynamic graphs at different time steps on Electricity dataset.} 
\label{dynamic_graph}  
\end{figure}

To capture the important node in the graph, we can discard the features node by node and evaluate the ground truth of the node's missing impact by Eq. \ref{missingEffect}. Then, we optimize Eq. \ref{FinalExpectation} to obtain the estimated  $p(\mathbf{P},\mathbf{Y}|\mathbf{Z})$ distribution. To evaluate the performance of the algorithm, we compare the ground truth to the estimated value on the PeMS-BAY and Electricity datasets in the block missing scenarios. The results are shown in Fig. \ref{Node_impact}. We arrange the horizontal coordination by the ground truth node missing impact sequence. The outcomes reveal that the estimated probability distribution $p(\mathbf{P},\mathbf{Y}|\mathbf{Z})$ can indicate the ground truth missing impact of nodes with the same increasing overall tendency, which offers valuable information to address missing datasets and holds significance for subsequent sensor maintenance efforts. 

\begin{figure}[!hbt]  
\subfloat[PeMS-BAY.]{ 
\includegraphics[width=\columnwidth]{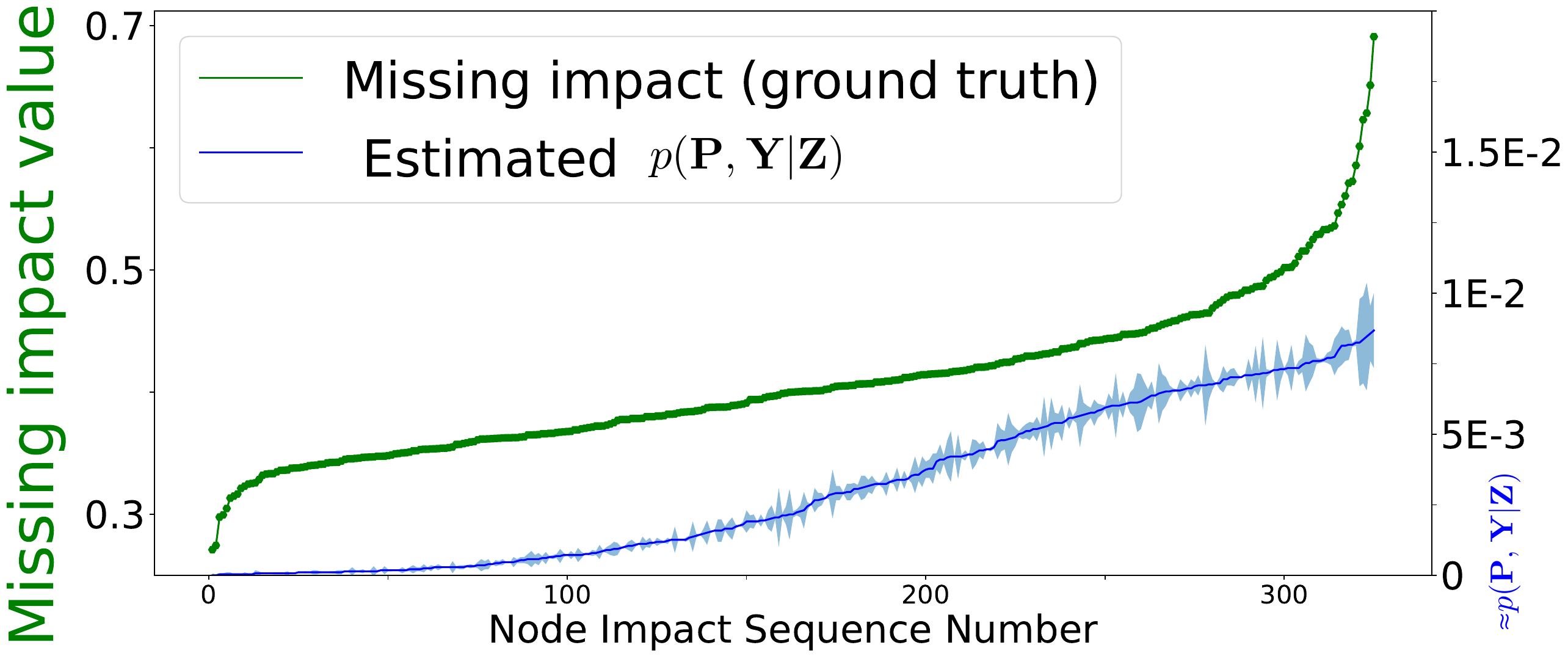} 
} 

\subfloat[Electricity.]{ 
\includegraphics[width=\columnwidth]{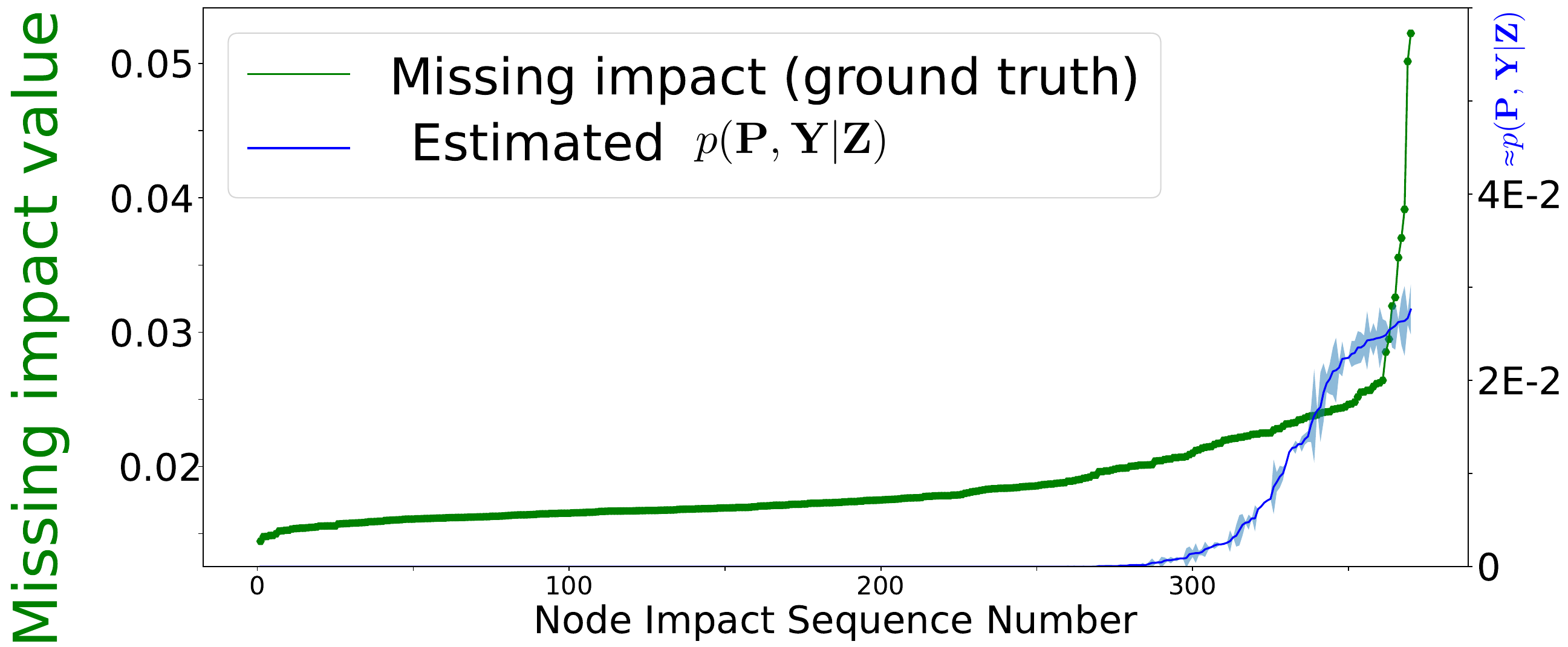}  
} 
\caption{The missing impact and  estimated $p(\mathbf{P},\mathbf{Y}|\mathbf{Z})$ distribution on PeMS-BAY and Electricity datasets .} 
\label{Node_impact}
\end{figure} 

To investigate node importance, we plot the sequences of the 50 most important nodes and corresponding node degrees for comparison. The results are presented in Fig. \ref{degree_graph}. As there is no ground truth of missing impact values in AQI dataset, we employ the estimated $p(\mathbf{P},\mathbf{Y}|\mathbf{Z})$ for comparison. Interestingly, the node with the highest degree does not hold the most significance within the entire graph neural network. Furthermore, nodes of equivalent degrees do not exert identical effects on the neural network's performance.   
\begin{figure}[!hbt]  
\centering  
\subfloat[AQI.]{ \centering 	 \includegraphics[width=0.9\columnwidth]{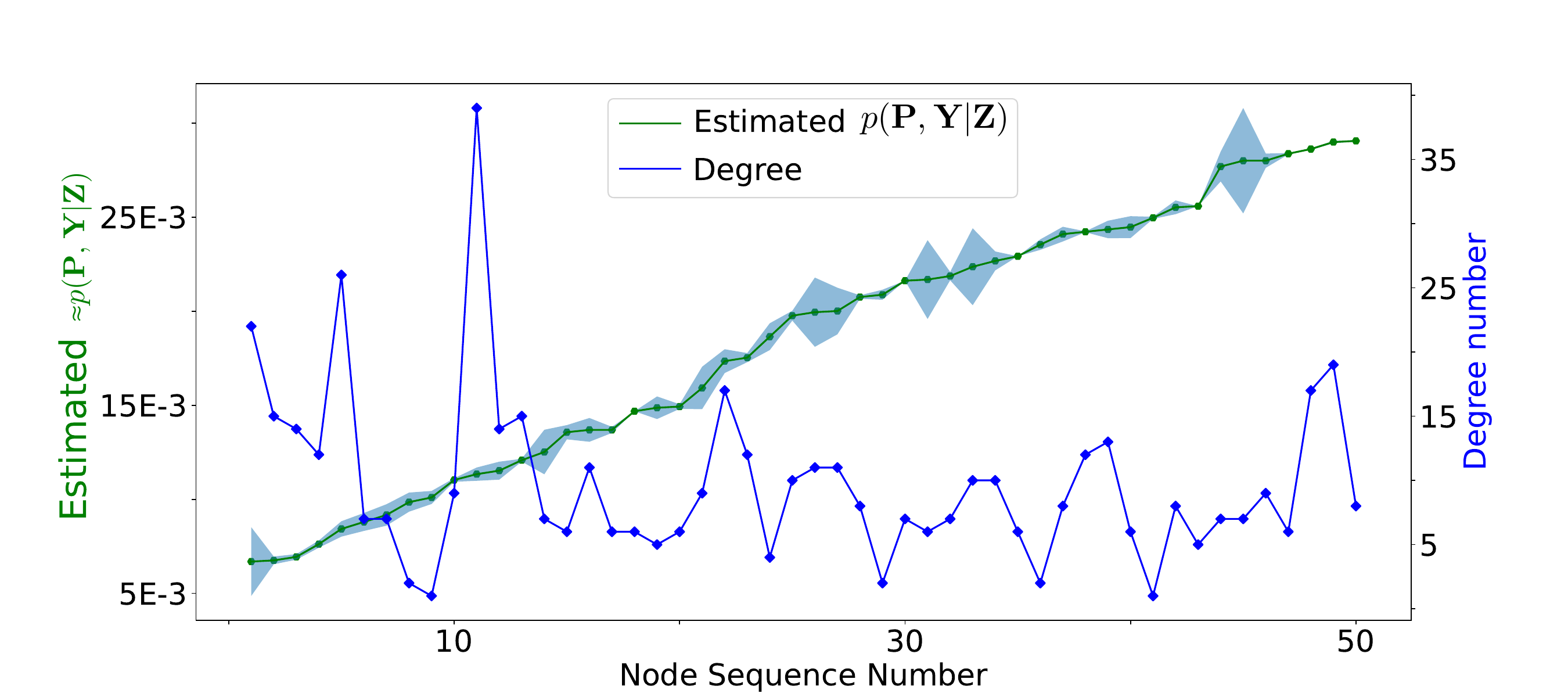} 	
\label{degree_graph1}   
}  

\subfloat[PeMS-BAY.]{  
\centering  \includegraphics[width=0.9\columnwidth]{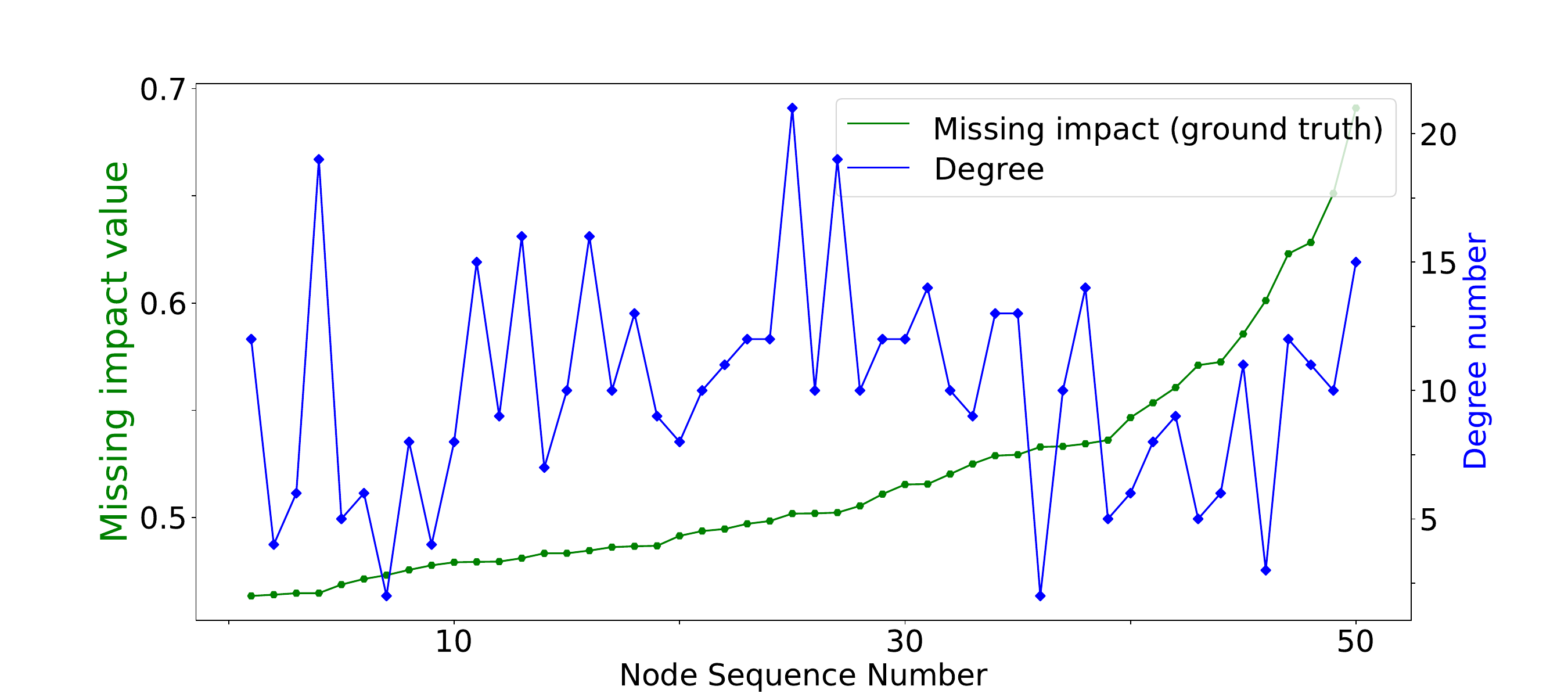}   \label{degree_graph2} 
} 
\caption{The node importance and degree of AQI and PeMS-BAY datasets.}  \label{degree_graph}  
\end{figure}

\section{Conclusion} 
In this paper, a novel HSPGNN is proposed for multivariate time series imputation. Different from earlier works, the model incorporated the physics model into the data-driven model, which makes the novel model more robust and explainable. Generic inhomogeneous PDE is adopted to constrain the spatio-temporal dependency on multi-order spatio-temporal correlation adaptively. Moreover, Chebyshev polynomials and 1D convolutional filters are utilized to reduce the space and time complexity by approximation. Also, we estimate the missing impact of NF to evaluate the importance of each node in the graph. Experiments demonstrate that our model can obtain better results than only data-driven models with better robustness in various time series datasets, especially in dealing with complex data missing situations. Also, dynamic graphs, the graph-like optical flow, and missing impact value can be derived naturally from this novel model with better explanations. In addition, the experiment shows that the combination of different-order neighbor nodes can obtain a better performance than only considering one specific hop since the real application scenario data are often governed by complex physics laws. In the future, we plan to explore how to further reduce the time complexity and methods for dealing with non-linear relationships in more complex scenarios.

\nocite{langley00}

\bibliography{HSP}
\bibliographystyle{icml2024}

\newpage
\appendix
\onecolumn
\section{Graph convolutional network} \label{GCN}
Graph convolutional network (GCN) builds a graph-based neural network model to graph representation. Laplacian matrix can be formulated as $\mathbf{L}=\mathbf{D}^{-\frac{1}{2}} (\mathbf{I}-\mathbf{A})\mathbf{D}^{\frac{1}{2}}$, where $\mathbf{I}$ is the identity matrix, which means adding self-connections, $\mathbf{D}$ is the degree matrix. For GCN, considering $K$-hop aggregation, GCN model can be expressed by the following formula \cite{kipf2016semi}: 
\begin{equation}
\label{eq4} \mathbf{H}^{(l+1)}=f\left(\mathbf{A}, \mathbf{H}^{(l)}\right)=\sigma\left(\sum\limits_{k=1}\limits^{K}  \mathbf{L}^{k} \mathbf{H}^{(l)} \mathbf{W}_k^{(l)}\right),
\end{equation}
where $\mathbf{H}^{(l)} \in \mathbb{R}^{N \times F}$ is the input of $l$ layer, while $\mathbf{H}^{(l+1)} \in \mathbb{R}^{N \times E}$ is output of $l$ layer with $E$ embedding, $k$ denotes the number of hops aggregation and $k \in \{1, ... ,K\}$. In addition, $\mathbf{W}_k \in \mathbb{R}^{F \times E}$ is a learnable parameter, and $\sigma$ represents the nonlinear activation function. In this study, we ignore the activation function and consider the feature matrix, we deploy the transpose operation to Eq. \ref{eq4}. Thus, the transpose GCN can be obtained as follows:
\begin{equation} 
\begin{aligned}
\label{gcn} 
\mathbf{X}^{(l+1)}=f\left(\mathbf{A}, \mathbf{X}^{(l)}\right)=\sum\limits_{k=1}\limits^{K} \mathbf{\Theta}_k^{(l)} \mathbf{X}^{(l)}\mathbf{L}^{k}, \quad \text{where} \quad \mathbf{\Theta}_k^{(l)} = (\mathbf{W}_k^{(l)})^T.
\end{aligned}
\end{equation}

\section{Spatial attention} \label{spatial attention}
Spatial attention (SAtt) \citep{feng2017effective} is widely adopted in time series prediction models. It can capture the dynamic spatial relationships effectively. The mechanism can be formulated as: 
\begin{equation}
\label{spatialAtt}  
\mathbf{S}_t=\mathbf{V}_s \sigma (((\mathbf{U}_{t}+\mathbf{\Bar{P}}_t)^T \mathbf{W}_1)\mathbf{W_2} (\mathbf{W_3}(\mathbf{U}_{t}+\mathbf{\Bar{P}}_t)+\mathbf{b}_s)  \end{equation}    
\begin{equation} 
\label{spatialAtt2}  (\mathbf{S}_t^{'})_{i,j} = \frac{\exp((\mathbf{S}_t)_{i,j})}{\sum_{j=1}^{N}\exp((\mathbf{S}_t)_{i,j})} 
\end{equation}   

where $\mathbf{V}_s$, $\mathbf{b}_s \in \mathbb{R}^{N \times N }$, $\mathbf{W}_1 \in \mathbb{R}^{1}$, $\mathbf{W}_2 \in \mathbb{R}^{1}$, $\mathbf{W}_3 \in \mathbb{R}^{1}$ are learnable parameters and $\sigma(.)$ is the activation function. The spatial semantic dependencies can be represented by the value of $\mathbf{S}_t$ in Equation \ref{spatialAtt}. In addition, $\mathbf{S}_t^{'}$ in Eq. \ref{spatialAtt2} is a spatial attention matrix normalized by the softmax function. Then, we can obtain the Laplacian matrix $\mathbf{L}_t$ with the spatial attention matrix $\mathbf{S}_t' \in \mathbb{R}^{N \times N}$ to dynamically adjust the weight adjacency matrix $\mathbf{A}_t$ at time $t$.

\section{LSTM}\label{LSTM}
 LSTM unit consists of a forget gate, an input gate, and an output gate. The forget gate determines which information to discard from the previous state. Similarly, the input gate decides which information to store in the current state using the same system as the forget gate. The output gate decides the output by taking the current state and previous states into consideration. By filtering the relevant information from the states in different time steps, the LSTM network can capture long- and short-term dependencies for making predictions of future time steps. The following equations can explain the mechanism of operation of the LSTM unit \citep{LSTM}: 
\begin{equation} 
\centering \mathbf{i}_t=\sigma(\mathbf{W}_i \mathbf{X}_t + \mathbf{U}_i\mathbf{h}_{t-1}+\mathbf{b}_i) \end{equation} \begin{equation} \centering \mathbf{f}_t=\sigma(\mathbf{W}_f \mathbf{X}_t + \mathbf{U}_f\mathbf{h}_{t-1}+\mathbf{b}_f)  \end{equation}  \begin{equation} \centering \tilde{\mathbf{c}}_t=\tanh(\mathbf{W}_c \mathbf{X}_t + \mathbf{U}_c\mathbf{h}_{t-1}+\mathbf{b}_c) \end{equation}  \begin{equation}  \centering \mathbf{c}_t=\mathbf{f}_t \odot  \mathbf{c}_{t-1} + \mathbf{i}_t \odot \tilde{\mathbf{c}}_t \end{equation}   \begin{equation}  \centering \mathbf{o}_t=\sigma(\mathbf{W}_o \mathbf{X}_t +\mathbf{U}_o \mathbf{h}_{t-1}+\mathbf{b}_o)  \end{equation}   \begin{equation} \centering \mathbf{h}_t=\mathbf{o}_t \odot \tanh(\mathbf{c}_t) 
\end{equation} 
where $\sigma$ is the activation function, $\mathbf{h}_t$ is the hidden features, $\mathbf{f}_t$ is the forget gate, $\mathbf{i}_t$ is the input gate, and $\mathbf{o}_t$ is the output gate, while $\mathbf{c}_t$ and $\mathbf{c}_{t-1}$ are the LSTM unit states. In addition, $\mathbf{W}_i$, $\mathbf{W}_f $, $\mathbf{W}_c$, $\mathbf{W}_o$, and $\mathbf{U}_i$, $\mathbf{U}_f $, $\mathbf{U}_c$, $\mathbf{U}_o$ are their parameters, while $\mathbf{b}_i$, $\mathbf{b}_f $, $\mathbf{b}_c$, and $\mathbf{b}_o$ are their bias in the embedding functions. Thus, we can obtain the future embedding states as follows: \begin{equation}         \mathbf{h}_{t+1:t+M}=\mathrm{LSTM}(\hat{\mathbf{P}}_{t-M:t}+\mathbf{U}_{t-M:t})  
\end{equation}

\section{Temporal attention} \label{Temporal attention}
To better capture the temporal features, the temporal attention mechanism is applied to the predictive stage. Attention is a powerful tool, which has been widely adopted by many famous models. It is proven that temporal attention can effectively explore the temporal correlations under different situations in different time slices adaptively. The mechanism of its operation can be expressed as \citep{guo2019attention}: 
\begin{equation}  
\label{eq11} 
\mathbf{E}=\mathbf{V}_e \sigma \big((\mathbf{h}_{t+1:t+M} \mathbf{U}_1)\mathbf{U_2} (\mathbf{U_3}\mathbf{h}_{t+1:t+M}^{T}+\mathbf{b}_e)\big)
\end{equation}   
\begin{equation}   
\label{eq12} \mathbf{E}_{i,j}^{'} = \frac{\exp(\mathbf{E}_{i,j})}{\sum_{j=1}^{M}\exp(\mathbf{E}_{i,j})}  \end{equation}   
where $\mathbf{V}_e$, $\mathbf{b}_e \in \mathbb{R}^{M \times M }$, $\mathbf{U}_1 \in \mathbb{R}^{N \times 1}$, $\mathbf{U}_2 \in \mathbb{R}^{1 \times N}$, $\mathbf{U}_3 \in \mathbb{R}^{1}$ are learnable parameters, $\sigma(.)$ is the activation function, and $M$ is the length of the temporal dimension, $\mathbf{h}$ represents the output of LSTM layers, the temporal semantic dependencies can be represented by the value of $\mathbf{E}_{i,j}$ by Equation \ref{eq11}. In addition, $\mathbf{E}^{'}$ is a normalized temporal attention matrix by the softmax function. Then, the output of temporal attention is calculated as $\mathbf{\mathbf{\hat{X}}}_{t+1:t+M}=\mathbf{E}^{'}\mathbf{h}_{t+1:t+M}$. 

\section{Overall algorithm procedure of HSPGNN}\label{HSPGNNAlgorithm}
The overall procedure for HSPGNN is summarized in Algorithm \ref{HSPA}.
\begin{algorithm}[!hbt]  
\caption{HSPGNN Algorithm}  
\label{alg:example}
\begin{algorithmic}  
\STATE {\bfseries Input:} Time series values $\mathbf{X}$,  mask $\mathbf{M}$   
\STATE {\bfseries Output:} The predicted missing values $\mathbf{\hat{P}}$, dynamic Laplacian matrix $\mathbf{L}^t$ and model parameters $\mathbf{\Theta}$, $\mathbf{W_v}$ and $\mathbf{\lambda}^{'}$ 
\STATE Randomly initialize the model parameters and input the available values $\mathbf{U} = \mathbf{X} \odot (1-\mathbf{M})$ in the training stage. 
\FOR{each epoch}
\FOR{each batch of $\mathbf{U}_{t-M:t}$ and mask $\mathbf{M}_{t-M:t}$}
\STATE Estimate the missing values by $\mathbf{\bar{P}}_{t-M:t}=\text{MLP}(\mathbf{U}_{t-M:t})$.
\STATE Calculate the dynamic Laplacian matrix $\mathbf{L}^t$ by spatial attention through Eq. \ref{spatialAtt} and \ref{spatialAtt2}.
\STATE Calculate the missing values with more accuracy $\mathbf{\hat{P}}_{t-M:t}$ by Eq. \ref{tmatrix}.
\STATE Input $\mathbf{\hat{P}}_{t-M:t}$ to the objective function Eq. \ref{objectiveFunction} and optimize the model parameters.
\ENDFOR   
\STATE Evaluates the performance by the validation set.
\ENDFOR
\STATE Choose the model parameters by the best performance from the validation set.
\STATE Return the predicted missing values $\mathbf{\hat{P}}$, dynamic Laplacian matrix $\mathbf{L}^t$ and model parameters.
\label{HSPA}
\end{algorithmic}
\end{algorithm}

\section{Experimental results} \label{extra_results}
\subsection{Datasets description}\label{dataset}
\begin{itemize}     
\item Air Quality (AQI): This dataset is a benchmark for the time series imputation from the Urban Computing project of 437 air quality monitoring stations over 43 cities in China. These stations collected six pollutants hourly from May 2014 to April 2015 with a high missing rate (25.67\%), including various missing patterns in real circumstances. In this study, only $PM2.5$ pollutant is taken into consideration. Based on the previous work \citep{yi2016st}, AQI-36 is a simpler version sub-dataset from AQI with a 13.24\% missing rate that only 36 sensor stations are adopted. Four months (March, June, September, and December) are used as the test set without overlapping with the training dataset.     
\item  PeMS-BAY: PeMS-BAY is a famous traffic dataset in spatio-temporal problems. The raw data are collected by Performance Measurement System (PeMS) in the San Francisco Bay Area for nearly 6 months (from Jan 1st 2017 to May 31th 2017) with 2\% originally missing rate and re-sampled into 5 minutes by previous work \citep{li2017diffusion}. The total road network contains 325 road sensors, and each sensor in the network can be regarded as a node in the graph. 80\% of the total time series is used for training, while 20\% for evaluation.    
\item Electricity: This is a public dataset that is released by UCI \citep{misc_electricityloaddiagrams20112014_321}. The electric power consumption of 370 clients is sampled every 15 minutes without missing data. In this study, 48 months (from Jan 1st 2011 to Dec 31st 2014) are chosen. Notably, the electricity dataset does not offer geographic relationships, contrary to the AQI and PeMS-BAY datasets. Similar to the AQI dataset, we apply the data ranging from Jan 2011 to Nov 2011 as the test set while the left is the training set, which is the same evaluation standard of \citet{du2023saits}. 
\end{itemize} 
To better consider the performance of our method, we only consider the out-of-sample scenarios in all datasets, which means the training and evaluation sequences are disjointed. 

\subsection{Data preprocessing and data enhancement}
According to Eq. \ref{Eq:problem definition1}, we initialize missing values as zero, including both the original missing values and the emulated missing values. Consequently, the input tensors inevitably become very sparse. To achieve better imputation results, we perform simple linear interpolation on the missing values node by node using the '$interp1d()$' command.  Additionally, concerning the PeMS-BAY and AQI datasets, we adopt data enhancement method to expand the training datasets due to the short length of the time series. This method repeats the data preprocessing approach. For the interpolated dataset from the data preprocessing step, we randomly drop out some values and interpolate the missing values using the same '$interp1d()$' command. By repeating this processing twice, three training datasets are obtained. Subsequently, we concatenate the three training sets into the enhanced training datasets using the '$np.concatenate(training set1, training set2, training set3)$' command.  Finally, the enhanced training datasets are input into the HSPGNN for training.

\subsection{Evaluation metrics} 
To facilitate comparison with baseline methods, we adopt the same metrics as used in previous baseline studies as follows: 
\begin{itemize}   
\item MAE (Mean Absolute Error)   
\begin{equation}      
MAE = \frac{\sum_{j=1}^{T} \sum_{i=1}^N|(\mathbf{X}_{ij} - \hat{\mathbf{X}}_{ij}) \odot \mathbf{M}_{ij}|}{ \sum_{j=1}^{T} \sum_{i=1}^N \mathbf{M}_{ij}}      
\end{equation}   
\item MSE (Mean Square Error)     
\begin{equation}      
MSE = \frac{\sum_{j=1}^{T} \sum_{i=1}^N(\mathbf{X}_{ij} - \hat{\mathbf{X}}_{ij})^2 \odot \mathbf{M}_{ij}}{ \sum_{j=1}^{T} \sum_{i=1}^N \mathbf{M}_{ij}}    
\end{equation} 
\end{itemize} 
where $\mathbf{X}_{ij}$ and $\hat{\mathbf{X}}_{ij}$ represent the ground truth values and imputed values in the $j$-th time step of the $i$-th node. $T$ is the total time step of samples in the test dataset; $N$ is the total number of nodes. Specifically, $MAE$ and $MSE$ are used to measure the imputation error, the smaller the value is, the better the imputation.

\subsection{Baseline}
We compare the performance of the HSPGNN model with the following baseline methods:
\begin{itemize}
    \item Mean: impute the missing value with the mean value of the total time range node by node.
    \item KNN \citep{cover1967nearest}: impute the missing value by the mean value of 10 nearest neighbor nodes.
    \item MF \citep{cichocki2009fast}: adopt the SVD method that calculates low-rank representations and uses them to recover the missing data.
    \item MICE \citep{white2011multiple}: deploy multiple imputations by chained statistical equations, limiting the maximum number of iterations to 100 and the number of nearest features to 10.
    \item MPGRU \citep{li2017diffusion}: combine the GNN with GRU for spatio-temporal data imputation. 
    \item Transformer \citep{vaswani2017advances}: find out the missing data by multi-head attention mechanism. 
    \item BRITS \citep{cao2018brits}: impute the missing value by bidirectional RNN structure.
    \item GRIN \citep{andrea2021filling}: adopt GNN and GRU for two-step missing data imputation in a bidirectional way. 
    \item M$^2$DMTF \citep{fan2021multi}: recover the missing data by multi-mode deep matrix and tensor factorization methods.  
    \item SAITS \citep{du2023saits}: apply self-attention mechanism for missing value imputation of time series.
    \item NHITS \citep{challu2023nhits}: predict the missing value by hierarchical interpolation and multi-rate data sampling using the MLP block. 
    \item TDM \citep{zhao2023transformed}: impute the missing values of two batches of samples by transforming them into a latent space through deep invertible functions and matching them distributionally.
\end{itemize}

\subsection{Graph-like Optical flow}
As for Optical Flow Based Moving Object Detection in computer vision domain, provided that in the neighborhood of pixel, change of brightness intensity $I(x,y,t)$ does not happen motion field and ignore the high order term, we can use the following expression the moving object can be formulated as follows \cite{aslani2013optical}:
\begin{equation}
    \triangledown I(x,y,t).V_t = -\frac{\partial I(x,y,t)}{\partial t}
\end{equation}
Where $\triangledown I(x,y,t)$ is so-called the spatial gradient of brightness intensity and $V_t$ is the optical flow(velocity vector) of the image pixel and $\frac{\partial I(x,y,t)}{\partial t}$ is the time derivative of the brightness intensity
the brightness intensity on $x, y$ pixel at time $t$. Optical flow can also be defined as the distribution of apparent velocities of movement of brightness pattern in an image.
Similarly, ignoring the higher-order term and the external source, we can obtain a similar graph-like optical flow from Eq. \ref{matrix} as follows:
\begin{equation}
       \mathbf{V}^{ij}_t=-\frac{ \mathbf{X}^i_t - \mathbf{X}^i_{t-1}}{\mathbf{X}_t^i - \mathbf{X}_{t}^j}=-\frac{\lambda }{\mathbf{\Theta}_1 \mathbf{L}_t^{ij}},
\end{equation}
where $\mathbf{V}^{ij}_t$ represent graph-like optical flow from $i$-th node to $j$-th node at time t while $\lambda$, $\mathbf{\Theta_1}$ and $\mathbf{\mathbf{L}_t^{ij}}$ are the parameters of the neural network and Laplacian matrix. Thus, we can obtain the transfer speeds between the nodes.

\subsection{Ablation study}\label{ablation1}
To verify the effectiveness of the individual components in HSPGNN, four variants of HSPGNN are made to compare to the best results of HSPGNN* (the best models on corresponding datasets): 
\begin{itemize}
    \item Without SAtt: retain all the components of our model except the spatial attention to generate the dynamic graphs. Notably, the Electricity dataset adopts the identity matrix as the adjacent matrix since there is no pre-defined graph.
    \item Without Physics Layer: Totally discard the physics-incorporated layers but retain all other components.
    \item Without LSTM and TAtt: Totally discard the LSTM and Temporal attention components but retain all other components. Notably, the synthetic value of 1-D interpolation is used in the training stage.
    \item Without MLP: Totally discard the MLP component but retain all other components.
    \item Vanilla PINN: Apply the physics PDE as part of the regularization, retain other components except for the physics-incorporated layers, and optimize the framework. 
\end{itemize}

We performed ablation experiments for the above variants on all the datasets (Block missing scenarios in PeMS-BAY and Electricity datasets). Tab. \ref{ablation} shows the measurements of MAE. It can be seen that the performance of our HSPGNN model is better than other variants, which confirms the effectiveness of each component in our model. Moreover, compared vanilla PINN and HSPGNN to the model without the physics layer, the results indicate that the physics law can really improve the performance, but the vanilla PINN model is less accurate than HSPGNN and exhibits large variability since it serves the physics constraint as the regularization term, which can not fully integrate the physics law and neural network effectively. In addition, how to tune the best regularization parameter is inevitable in the vanilla PINN framework, which in turn makes it difficult in real applications. However, as illustrated in Fig. \ref{training}, the incorporation of physics into GNNs circumvents this issue, facilitating smoother convergence by seamlessly integrating PDE constraints into the neural network architecture.
\begin{table}[]
\centering 
\caption{The result of ablation study.} 
{
\begin{tabular}{lllll} \hline \multicolumn{1}{c|}{\multirow{2}{*}{Datasets}} & PeMS-BAY           & Electricity        & AQI                 & AQI-36              \\ \cline{2-5}                            \multicolumn{1}{c|}{}& MAE                & MAE                & MAE                 & MAE                 \\ \hline \textbf{HSPGNN}*                      & \textbf{1.10±0.02} & \textbf{0.15±0.01} & \textbf{12.85±0.12} & \textbf{11.19±0.20} \\ Without SAtt              & 1.27±0.03          & 1.67±0.20          & 12.95±0.14          & 12.38±0.23          \\ Without Physics Layer     & 1.25±0.02          & 0.17±0.01          & 13.43±0.18          & 12.73±0.27          \\ Without LSTM and TAtt     & 1.20±0.02          & 0.18±0.01          & 13.69±0.16          & 12.31±0.26          \\ Without MLP               & 1.34±0.02          & 0.17±0.01          & 13.92±0.18          & 12.55±0.25          \\  Vanilla PINN               & 1.24±0.06          & 0.16±0.01          & 13.33±0.22          & 12.27±0.46          \\  \hline  
\label{ablation}
\end{tabular}}  
\end{table}

\begin{figure}[!hbt]  
\centering  
\subfloat[PeMS-BAY.]{
\includegraphics[width=0.25\columnwidth]{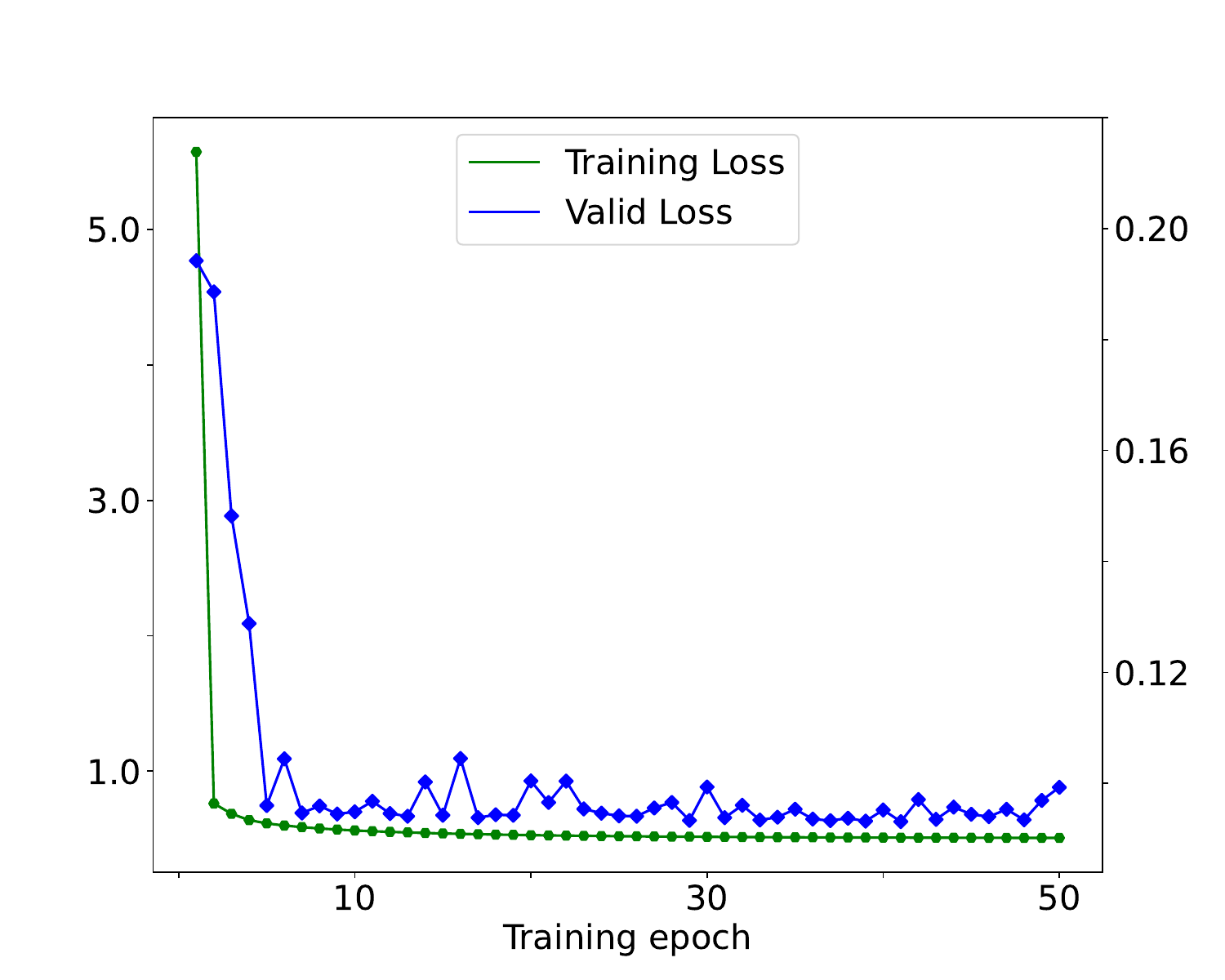} 	
\label{training1} }   
\subfloat[Electricity.]{  
\includegraphics[width=0.25\columnwidth]{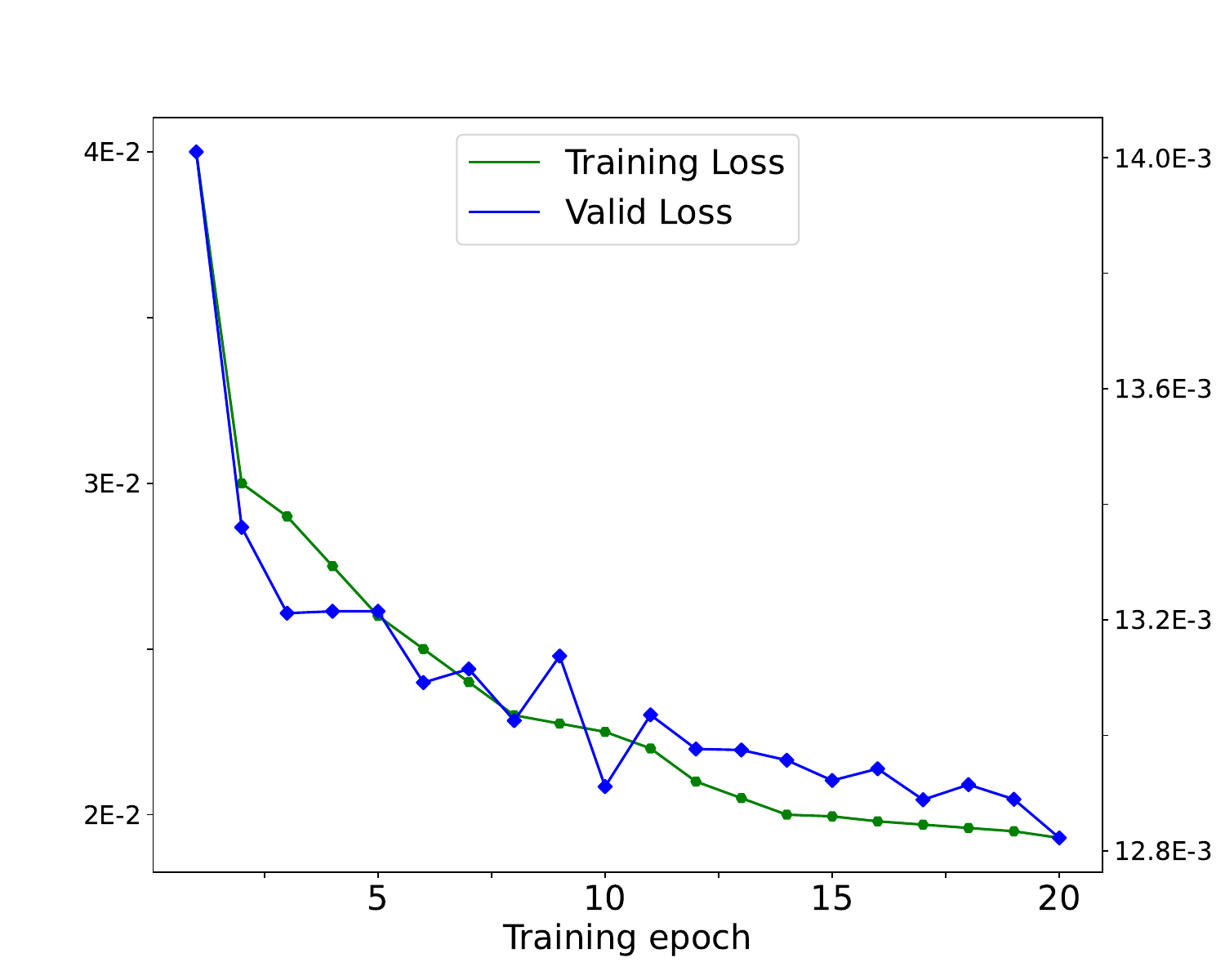}  
\label{training2}  }

\subfloat[AQI dataset.]{
\includegraphics[width=0.25\columnwidth]{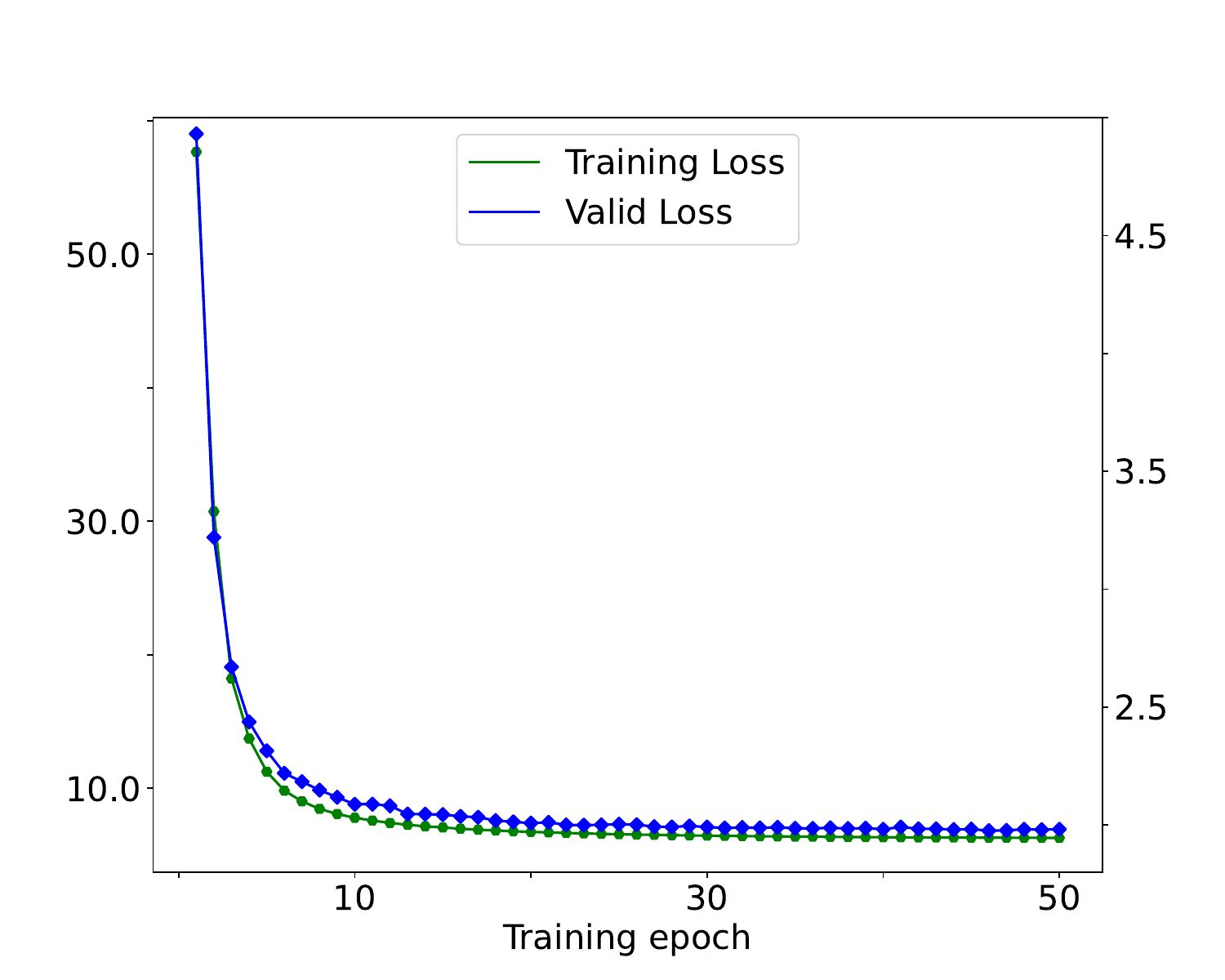} 	
\label{training3} }    
\subfloat[AQI36 dataset.]{ 
\includegraphics[width=0.25\columnwidth]{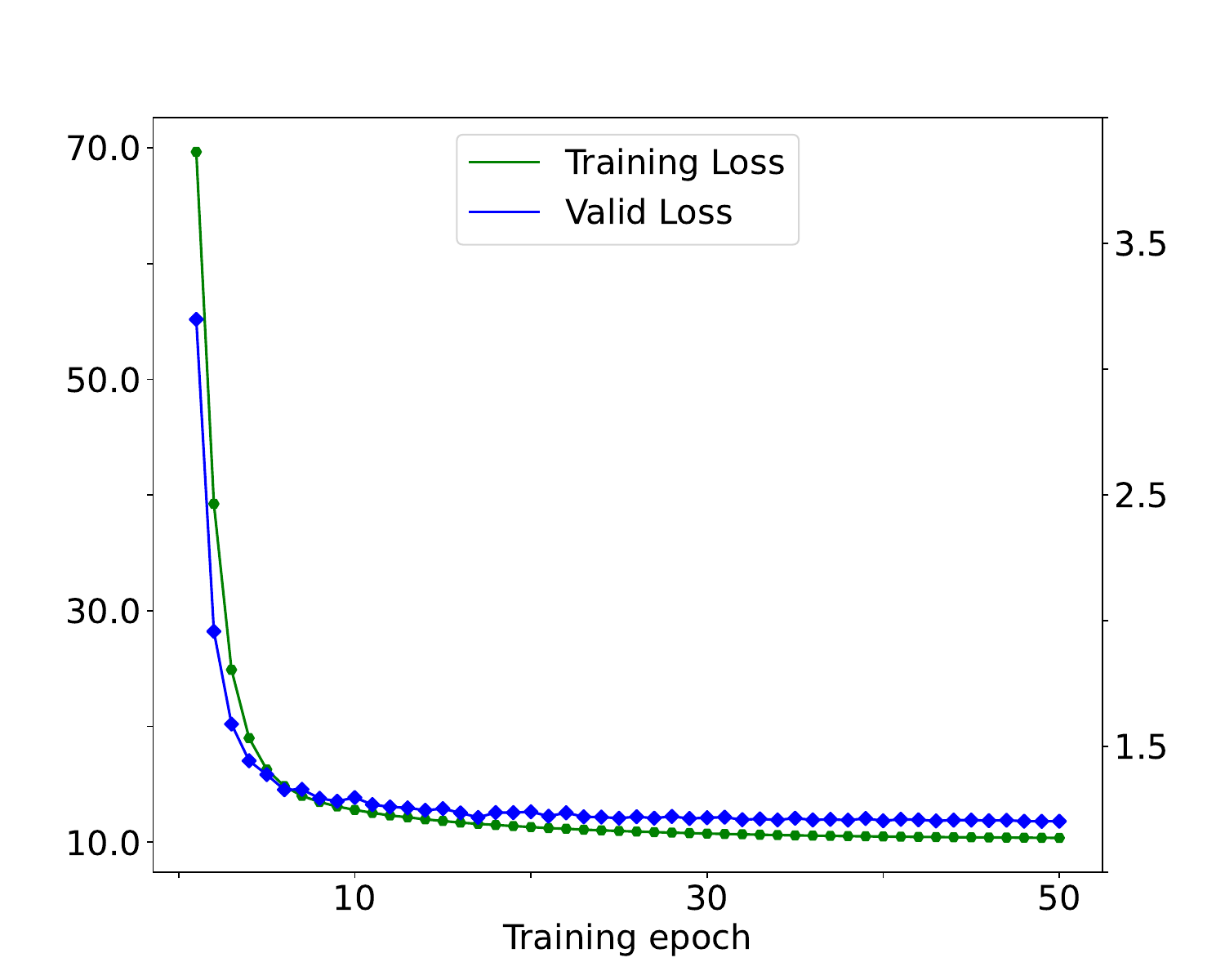}   
\label{training4}  }
\caption{The training and validation loss of four datasets.} 
\label{training}  
\end{figure}

\end{document}